\documentclass[10pt]{article} 
\usepackage[accepted]{tmlr}

\usepackage{amsmath,amsfonts,bm}

\def\eqref#1{equation~\ref{#1}}

\def\1{\bm{1}}

\def\rh{{\textnormal{h}}}

\DeclareMathAlphabet{\mathsfit}{\encodingdefault}{\sfdefault}{m}{sl}
\SetMathAlphabet{\mathsfit}{bold}{\encodingdefault}{\sfdefault}{bx}{n}

\usepackage{hyperref}
\usepackage{url}
\usepackage{amsmath, amssymb, amsthm, mathtools}
\usepackage{enumitem}
\usepackage{tikz}
\usetikzlibrary{arrows.meta, calc}

\newtheorem{theorem}{Theorem}
\newtheorem{lemma}{Lemma}
\newtheorem{definition}{Definition}
\newtheorem{remark}{Remark}
\newtheorem{proposition}{Proposition}
\newtheorem{corollary}{Corollary}

\DeclareMathOperator{\emb}{emb}
\DeclareMathOperator{\out}{out}

\newcommand{\Z}{\mathbb Z}

\newcommand{\Aff}{\operatorname{Aff}}
\newcommand{\F}{\mathbb{F}}

\title{On the Expressive Power and Limitations of Multi-Layer SSMs}

\author{\name Nikola Zubić \email zubic@ifi.uzh.ch \\
      \addr Robotics and Perception Group \\
      University of Zurich
      \AND
      \name Qian Li \email liqian.ict@gmail.com \\
      \addr Shenzhen International Center For Industrial And Applied Mathematics \\
      Shenzhen Research Institute of Big Data
      \AND
      \name Yuyi Wang \email yuyiwang920@gmail.com \\
      \addr Tengen Intelligence Institute \\
      CRRC Zhuzhou Institute
      \AND
      \name Davide Scaramuzza \email sdavide@ifi.uzh.ch \\
      \addr Robotics and Perception Group \\
      University of Zurich}

\def\month{09}  
\def\year{2026} 
\def\openreview{\url{https://openreview.net/forum?id=NQ8IcsIti8}} 

\begin{document}

\maketitle

\begingroup
\renewcommand{\thefootnote}{\fnsymbol{footnote}}
\footnotetext[2]{This paper received the TMLR Journal-to-Conference (J2C) Certification and was accepted as a paper at NeurIPS 2026 through the Journal-to-Conference Track.}
\endgroup

\begin{abstract}
We study how depth, finite precision, state dimension, and chain-of-thought
(CoT) affect the expressive power of multi-layer state-space models (SSMs).
For the explicit-table $K$-function-composition problem, a canonical benchmark
for sequential information propagation, we prove that any $L$-layer SSM solving
$(L+3)$-function composition must satisfy
$d^2p=\Omega(N/L^3)$, where $d$ is the state dimension and $p$ is the
per-scalar precision. Conversely, $K$-function composition is solved exactly
by a $(K+1)$-layer generalized SSM with $d=1$ and
$p=\Theta(\log N)$. This gives a worst-case depth hierarchy for this formal
problem family. We then distinguish post-input reasoning, in which all thought tokens are generated after the input, from input-interleaved reasoning, in which thought tokens may be inserted while the input stream is being read. Post-input reasoning does not circumvent our communication-based lower-bound pipeline, whereas input-interleaved reasoning admits bidirectional simulations with general deterministic one-pass streaming algorithms at the granularity of persistent memory. Finally, width and precision are not interchangeable under exact step-preserving simulation in the base affine-state model, but become interchangeable through the streaming-memory characterization once input-interleaved reasoning is allowed.
\end{abstract}

\section{Introduction}

State-space models (SSMs) have emerged as a promising alternative to transformers for
sequence modeling, offering linear-time inference and principled mechanisms for capturing
long-range dependencies~\citep{gu2022s4,gu2023mamba,dao2024mamba2}.
Architectures such as S4~\citep{gu2022s4} and Mamba~\citep{gu2023mamba} process
sequences through a recurrence that is \emph{linear} in the hidden state yet
\emph{input-dependent} in its transition parameters, enabling efficient parallel
training via associative scans while retaining the streaming efficiency of recurrent
models. These models have achieved strong empirical performance across language, audio, and genomics, and their multi-layer variants are now deployed at scales comparable to transformer-based large language models.

Despite this practical success, a rigorous understanding of the \emph{expressive power}
of multi-layer SSMs remains incomplete. A growing body of theoretical work has begun to
map the computational landscape of these architectures. \citet{merrill2024illusion}
showed that, under standard complexity-theoretic assumptions, common SSMs like S4 and Mamba
cannot express computations outside $\mathsf{TC}^0$, placing them on a similar footing to
transformers in terms of circuit complexity. \citet{sarrof2024expressive}
studied SSM expressiveness through the lens of formal languages, identifying both strengths
and weaknesses relative to transformers. \citet{cirone2024theoretical}
provided a continuous-time analysis via rough path theory, characterizing the closure of
linear controlled differential equations that underpin selective SSMs, while \citet{zubic2025regularity} studied regularity and stability properties of selective SSMs with discontinuous gating. Although these results
offer important insights, they primarily address \emph{single-layer} or
\emph{time-invariant} models, or operate in asymptotic regimes that do not directly
capture the interplay between \emph{depth}, \emph{finite precision}, and
\emph{state dimension} in multi-layer architectures.

In parallel, the role of \emph{chain-of-thought} (CoT) reasoning has led to numerous theoretical works in the context of transformers. \citet{merrill2024cot} showed that allowing a transformer decoder
to generate intermediate tokens before answering can fundamentally expand its
computational power: a linear number of CoT steps enables simulation of arbitrary
finite automata, while polynomial steps yield the full power of $\mathsf{P}$. \citet{li2024cot} proved analogous results for constant-precision
transformers, connecting CoT to circuit size. \citet{chen2025theoretical} studied the computational power of Transformers without CoT and the function composition problem. These findings raise a natural question
for SSMs: \emph{does CoT similarly amplify the power of SSMs,
and if so, does the timing of CoT generation matter?}

\paragraph{This work.}
We provide a unified theoretical analysis of the expressive power and limitations of
multi-layer SSMs, organized around three axes: \emph{compositional lower bounds},
\emph{the role of CoT}, and \emph{width--precision tradeoffs}.
Our results are summarized as follows:
\begin{enumerate}[leftmargin=2em, label=(\roman*)]
\item \textbf{Lower bound via communication complexity (Theorem~\ref{thm:main}).}
Under the canonical blockwise explicit-table encoding, we show that any
$L$-layer SSM solving the $(L{+}3)$-function-composition problem must satisfy
$d^2 p = \Omega(N/L^3)$, where $d$ is the state dimension, $p$ is the
per-scalar precision, and $N$ is the problem size. The proof introduces a \emph{forward communication model}
(Definition~\ref{def:arcomm}) that captures the layer-by-layer information flow in
multi-layer SSMs, and reduces to pointer chasing lower
bounds~\citep{nisan1993rounds,mao2025gadgetless}. This establishes a worst-case
resource separation on a canonical sequential-composition benchmark. It is not a
claim that every compositional task admits the same reduction or lower bound.

\item \textbf{Complementary upper bound and depth hierarchy.}
We complement the lower bound with
a construction showing that $K$-fold function composition can be solved exactly
by a $(K{+}1)$-layer SSM with $d=1$ and $p=\Theta(\log N)$, yielding $dp = O(\log N)$.
Specializing to $K=L+3$ gives a constant-gap depth hierarchy: the
$(L{+}3)$-composition problem is easy for $L{+}4$ layers, hence for $O(L)$ layers,
but hard for $L$ layers.

\item \textbf{Post-input reasoning does not remove the lower bound. Input-interleaved reasoning yields streaming equivalence
(Proposition~\ref{prop:offlinecot}, Theorem~\ref{thm:onlinecot-streaming}).}
We distinguish \emph{post-input reasoning} (offline CoT), in which all
self-generated tokens appear after the complete exogenous input, from
\emph{input-interleaved reasoning} (online CoT), in which such tokens may be
inserted between exogenous tokens. Post-input reasoning is local
post-processing in our communication reduction and therefore leaves this
lower-bound pipeline unchanged. This is not a claim that post-input
computation is useless for every task.
In contrast, within the deterministic generalized model, input-interleaved
reasoning admits bidirectional simulation with general one-pass streaming
algorithms at the granularity of persistent memory. Consequently, a
single-layer generalized SSM solves arbitrary-length function composition
with $dp=O(\log N)$ using one thought token per exogenous token
(Corollary~\ref{cor:cot-composition}).

\item \textbf{Width and precision are not interchangeable in the base model
(Theorems~\ref{thm:no-collapse} and~\ref{thm:no-reverse-collapse}).}
We prove that in the base (no-CoT) model, the product $dp$ is \emph{not} a complete
invariant of computational power: a width-$w$, precision-$p$ machine cannot, in general, be simulated by a width-$1$, precision-$pw$ machine, nor vice versa. The proofs are
algebraic, exploiting a counting argument over affine transition maps
(Theorem~\ref{thm:no-collapse}) and the nonexistence of order-$8$ affine permutations
over $\mathbb{F}_2^3$ (Theorem~\ref{thm:no-reverse-collapse}). However, once online
CoT is allowed, the correct invariant becomes the total persistent memory $Lwp$,
and width and precision become fully interchangeable
(Proposition~\ref{prop:cot-memory}).
\end{enumerate}

\paragraph{Techniques.}
Our lower bound strategy connects SSMs to multi-party communication via a \emph{forward communication model} in which $K$ players, each holding
one block of the input, communicate in $L$ synchronous rounds.
The key observation (Lemma~\ref{lem:ssm-to-comm}) is that because each SSM layer
performs an \emph{affine} state update, the effect of an entire input block on the hidden state can be summarized by a $(d^2{+}d)$-parameter affine map, which any downstream player can compose with its own summary. This reduction converts an
$L$-layer SSM into an $L$-round protocol with message length $O(d^2 p)$. An improved serialization of these synchronous rounds into $L+1$ alternating two-party messages, combined with the pointer chasing lower bound of~\citet{mao2025gadgetless},
yields the desired $\Omega(N/L^3)$ bound.

For the CoT results, the decisive issue is when thought tokens are generated
relative to the exogenous input stream. Post-input reasoning amounts to local
post-processing by the final player and cannot alter the information exchanged
during the communication rounds. Input-interleaved reasoning permits
state-dependent feedback before the next exogenous token arrives, enabling a
bidirectional simulation with general deterministic one-pass streaming
algorithms at the granularity of persistent memory.

\paragraph{Organization.}
Section~\ref{sec:related} surveys related work.
Section~\ref{sec:setup} introduces the generalized multi-layer SSM and
the communication model.
Section~\ref{sec:lower} presents the main lower bound and its proof strategy
via the communication reduction.
Section~\ref{sec:upper} gives the complementary upper bound via explicit construction.
Section~\ref{sec:cot} formalizes post-input (offline) and input-interleaved (online)
CoT and establishes their contrasting effects on expressiveness.
Section~\ref{sec:width-precision} analyzes width--precision tradeoffs.
Section~\ref{sec:discussion} discusses implications and open problems.
Full proofs are deferred to
Appendices~\ref{app:communication-proofs}--\ref{app:width-precision-proofs}. The main text retains theorem statements and proof roadmaps.

\section{Related Work}\label{sec:related}

\paragraph{State-space models and efficient recurrences.}
The structured state-space model S4~\citep{gu2022s4} demonstrated that linear recurrences with carefully parameterized transition matrices capture long-range dependencies while
admitting efficient parallel training. This paradigm has since
expanded into a rich family of architectures: S5~\citep{smith2023s5} simplifies to a
diagonal parameterization purely in the time domain while retaining parallelism, the Linear Recurrent
Unit~\citep{orvieto2023resurrecting} further distills the design to its minimal
components, Mamba~\citep{gu2023mamba} introduces input-dependent (selective) gating and
achieves transformer-competitive language modeling, S7~\citep{soydan2024s} makes S5 input-dependent, and
Mamba-2~\citep{dao2024mamba2} reveals a formal duality between selective SSMs and
structured attention. Additional efficient recurrent alternatives include
RWKV~\citep{peng2023rwkv}, Griffin~\citep{de2024griffin} and GG-SSMs~\citep{Zubic_2025_CVPR}.
A common structural feature of many of these models, and the only such feature
used by our communication reduction, is that the update is affine in the
previous hidden state once the current layer input is fixed. These affine updates can be composed over a token block. Definition~\ref{def:ssm} isolates this
property while deliberately abstracting away architecture-specific
parameterizations, gating constraints, training dynamics, and hardware
implementations. It is therefore a common upper-level model for lower-bound
analysis.

\paragraph{Expressiveness and limitations of recurrent and state-space models.}
The theoretical study of recurrent architectures has a long history. \citet{siegelmann1995computational} established that recurrent neural networks with infinite-precision rational weights are Turing-complete, but this result relies
crucially on unbounded precision. Under finite or saturated precision, the computational
power contracts sharply: \citet{weiss2018practical} showed empirically
that finite-precision RNNs behave as finite automata, and \citet{merrill2020formal} formalized this by proving that saturated RNNs
recognize exactly the regular or counter languages depending on the gating architecture.
For SSMs specifically, \citet{merrill2024illusion} proved
that standard architectures, including S4, Mamba, and RWKV, under log-precision
constraints have their outputs computable in the circuit class $\mathsf{TC}^0$. \citet{sarrof2024expressive} characterized SSM expressiveness through formal languages, identifying separations from transformers in both directions. On the empirical side, \citet{deletang2023neural} systematically tested neural architectures against the
Chomsky hierarchy, \citet{arora2024zoology} showed that SSMs struggle with associative recall compared to attention, \citet{jelassi2024repeat} demonstrated a significant gap on copying tasks, \citet{zubic2025limits} showed that single-layer SSMs struggle on the function composition task, and \citet{wen2024rnns} identified in-context retrieval as a key bottleneck
separating RNNs from transformers.

Our work departs from these prior results in two respects. First, whereas the above
characterize SSMs in terms of complexity or language class membership (e.g.,
$\mathsf{TC}^0$, regular languages), we provide \emph{quantitative} lower bounds on
the product $d^2 p$ as a function of the task parameter~$N$ and the layer count~$L$.
Second, our analysis is inherently \emph{multi-layer}: we show how the number of
layers creates a hierarchy for compositional tasks, a phenomenon that single-layer analyses cannot capture.

\paragraph{Depth separations in neural sequence models.}
Depth separation is a recurring theme in neural network theory. For feedforward
networks, \citet{telgarsky2016benefits} proved that depth yields exponential
gains in expressiveness over width. For transformers, \citet{merrill2022saturated} showed that constant-depth transformers under
saturated precision compute exactly the functions in $\mathsf{TC}^0$. Our work establishes an analogous depth hierarchy for SSMs, but through a
different proof strategy: rather than circuit simulation arguments, we reduce to
multi-round communication complexity via an autoregressive protocol model. The $K$-function composition task we employ is a natural benchmark for sequential depth,
closely related to the iterated function problems in the transformer and circuit
complexity literature~\citep{merrill2023parallelism}.

\paragraph{Communication complexity and pointer chasing.}
Multi-round communication complexity provides our primary technical tool. \citet{nisan1993rounds} initiated the study of round--communication tradeoffs
for the pointer chasing problem. Subsequent refinements by \citet{ponzio1999pointer} and \citet{yehudayoff2020pointer} tightened the
bounds. Most recently, \citet{mao2025gadgetless} proved the strongest known
lower bound in the form of $\Omega(N/K+K)$ via a gadgetless lifting technique, which we use as a
black box.
The use of communication complexity to derive space and streaming lower bounds is
classical~\citep{alon1999space,kushilevitz1997communication}. Our contribution is a new
\emph{forward communication model} (Definition~\ref{def:arcomm}) designed to
match the information flow in multi-layer SSMs. In this model, $K$~players hold
consecutive input blocks and communicate over $L$ synchronous rounds that mirror the
$L$~layers of the SSM. The key structural insight is that the affine recurrence in each
SSM layer allows a player to compress its block's effect into an $O(d^2 p)$-bit affine
summary, enabling a faithful simulation by a communication protocol.

\paragraph{Chain-of-thought reasoning.}
Chain-of-thought (CoT) prompting~\citep{wei2022chain} and the related scratchpad
mechanism~\citep{nye2021show} have been shown empirically to improve the reasoning
capabilities of large language models. On the theoretical side, \citet{feng2023towards} showed that intermediate reasoning steps enable transformers
to solve problems beyond their base expressiveness, \citet{merrill2024cot} proved that bounded-precision transformer decoders with
linearly many CoT steps simulate arbitrary finite automata, and with polynomially many
steps capture all of $\mathsf{P}$. \citet{li2024cot} connected CoT length to
circuit size, and \citet{huang2025cot} studied the learnability and length
generalization of CoT reasoning.

These theoretical results all concern \emph{transformers}. To our knowledge, our work
provides the first formal analysis of the CoT for \emph{SSMs}.
We identify a qualitative distinction between \emph{post-input} (offline)~CoT
(thought tokens generated only after the full input) and \emph{input-interleaved}
(online)~CoT (thought tokens interleaved during the input stream) that exploits the
sequential, autoregressive nature of SSMs.
Post-input CoT amounts to post-processing by the last player in the communication model
and cannot circumvent our lower bounds (Proposition~\ref{prop:offlinecot}). Input-interleaved CoT
fundamentally alters the information flow, allowing feedback from deeper layers to reach
earlier stream positions and collapsing the model to a universal one-pass streaming
simulator (Theorem~\ref{thm:onlinecot-streaming}). This online--offline dichotomy has
no direct counterpart in the transformer CoT literature, where attention over all prior
positions renders the distinction less consequential.

\paragraph{Width, precision, and resource tradeoffs.}
The interplay between state dimension and bit precision is implicit in finite-precision
analyses of recurrent models~\citep{merrill2024illusion,weiss2018practical}, where the effective capacity is governed by the total memory budget. However, to our knowledge,
the question of whether width and precision are \emph{interchangeable}, i.e., whether a
width-$w$, precision-$p$ machine can always be replaced by a width-$1$,
precision-$wp$ machine, has not been formally investigated. Our algebraic results
(Theorems~\ref{thm:no-collapse} and~\ref{thm:no-reverse-collapse}) show that in the
base affine-state model the answer is negative in both directions, due to the richer
algebraic structure of matrix-valued (as opposed to scalar) affine transitions. This non-interchangeability is resolved when online CoT is allowed, as the equivalence to one-pass streaming makes only the total bit budget relevant
(Proposition~\ref{prop:cot-memory}).

\section{Preliminaries}
\label{sec:setup}

We work with the following generalized multi-layer state space model (SSM).

\begin{definition}[Generalized multi-layer SSM]\label{def:ssm}
Fix $L\in\mathbb{N}$ and state dimension $d\in\mathbb{N}$. Let $\{x_t\}_{t\ge 1}\subseteq \mathbb{R}^m$ be an input sequence and set $y_{0,t}=\emb(x_t,t)$, where $\emb$ is any fixed function.
An $L$-layer SSM consists, for each layer $\ell\in\{1,\dots,L\}$ and each time $t\ge 1$, of matrices
$A_{\ell,t}\in\mathbb{R}^{d\times d}$ and linear maps $B_{\ell,t}$ such that $B_{\ell,t}y_{\ell-1,t}\in\mathbb{R}^d$,
together with a readout map $\out_{\ell,t}:\mathbb{R}^d\times \mathbb{R}^m\to\mathbb{R}^m$,
and hidden/output variables $h_{\ell,t}\in\mathbb{R}^d$, $y_{\ell,t}\in\mathbb{R}^m$ obeying
\begin{equation}
\label{eq:ssm}
h_{\ell,t} = A_{\ell,t} h_{\ell,t-1} + B_{\ell,t} y_{\ell-1,t},
\qquad
y_{\ell,t} = \out_{\ell,t}(h_{\ell,t},y_{\ell-1,t}).
\end{equation}
We allow $A_{\ell,t},B_{\ell,t}$ to depend on $\ell,t$, the input length $n$, and the current layer input $y_{\ell-1,t}$.
All real-valued parameters are represented in a fixed finite-precision encoding of $p$ bits per scalar.
The network output at time $t$ is $y_t := y_{L,t}$.
\end{definition}

\begin{remark}[Typing conventions]
We write that $B_{\ell,t}\in\mathbb{R}^{d\times d}$ and $y_{\ell,t}\in\mathbb{R}^m$.
To ensure the update $B_{\ell,t}y_{\ell-1,t}$ is well-defined, it suffices (and is common in the theory setting) to assume
either $m=d$ or $B_{\ell,t}$ has the appropriate shape $\mathbb{R}^{d\times m}$. The arguments below only use that
$b_{\ell,t}\coloneqq B_{\ell,t}y_{\ell-1,t}$ is a $d$-vector with $p$-bit entries, so the proof is insensitive to this choice.
\end{remark}

\begin{remark}[Scope of the abstraction]\label{rem:scope}
Definition~\ref{def:ssm} deliberately isolates the affine-in-state recurrence used by many
SSM architectures: conditional on the current layer input, each token induces
an affine map of the preceding hidden state. Consequently, a lower bound for
this generalized model also applies to a concrete architecture whenever that
architecture is faithfully represented as a subclass under the same
finite-precision resource accounting.

The converse should not be inferred. Our upper constructions are existential
constructions in the generalized model, whose input-dependent transitions,
readouts, and vector-valued internal tokens need not be realizable by a
particular structured parameterization such as S4 or Mamba. They also do not
establish learnability, numerical robustness, scan-parallel execution, or
comparable hardware cost.
\end{remark}

\subsection{Function composition task and communication models}

\begin{definition}[$K$-function composition under the canonical stream encoding]\label{def:funccomp}
Fix $N,K\in\mathbb{N}$. An instance consists of an initial element
$a\in[N]$ and functions
$
    f_1,f_2,\ldots,f_K:[N]\to[N].
$
Unless stated otherwise, we use the canonical blockwise explicit-table
encoding of length $n=1+KN$:
\[
    x_1:=a,
    \qquad
    x_{1+(i-1)N+j}:=f_i(j),
\]
for every $i\in\{1,\ldots,K\}$ and $j\in[N]$. Thus, the stream first
presents $a$ and then lists the lookup tables of
$f_1,\ldots,f_K$ consecutively, each in the order
$f_i(1),\ldots,f_i(N)$. Equivalently, the blocks used by the communication
reduction are
\[
    I_1=[1:N+1],
    \qquad
    I_i=[2+(i-1)N:1+iN],
    \quad i\in\{2,\ldots,K\}.
\]
The goal is to output
\[
    (f_K\circ f_{K-1}\circ\cdots\circ f_1)(a).
\]
An algorithm solves the task with error probability at most $1/3$ if it
outputs the correct value with probability at least $2/3$ over its internal
randomness. Deterministic algorithms are the special case with success
probability one.
\end{definition}

\paragraph{Why function composition?}
Function composition captures the basic operation of carrying the output of
one computation into the next. We do not claim that every compositional task
inherits the same lower bound. Its significance is instead that it is a
canonical benchmark for sequential information propagation: related
structures arise in iterated state transitions, pointer chasing, repeated
transformations, and multi-step algorithmic reasoning. Thus, the result
identifies a broader architectural obstruction---with a fixed SSM depth,
information cannot always be adaptively propagated through an arbitrarily
long chain---while the precise quantitative statement remains limited to our
formal model and problem family.

An arbitrary staged deterministic computation can be described abstractly as
a composition of transition maps. However, converting those maps into the
explicit-table representation above may require a domain exponentially larger
than the original working memory, increase the input length accordingly, or
fail to preserve the stream order required by our communication reduction.
Consequently, our lower bound transfers to another problem only when there is
a resource-preserving and stream-order-preserving reduction from function
composition or pointer chasing to that problem. No universal reduction is
claimed here.

\begin{definition}[Forward communication model]\label{def:arcomm}
There are $K$ players. Player $i$ holds the portion of the input corresponding to $f_i$, and player $1$ additionally holds $a$.
Communication proceeds in $L$ \emph{rounds} (epochs). In each round $\ell\in\{1,\dots,L\}$, every player $i$ sends a
message $M_{\ell,i}$ of at most $B$ bits to all players $j>i$.
Crucially, $M_{\ell,i}$ may depend on the player input and on the \emph{entire transcript of rounds $<\ell$}, but not on any
messages sent in the current round $\ell$ (a synchronous round).
At the end of round $L$, player $K$ must output the answer.
\end{definition}

\begin{figure}[t]
\centering
\begin{tikzpicture}[
  >={Stealth[length=2.2mm]},
  player/.style={draw=black!70, rounded corners=2pt, minimum width=1.1cm,
                 minimum height=0.6cm, fill=blue!8, font=\small},
  outnode/.style={draw=black!70, rounded corners=2pt, minimum width=1.5cm,
                  minimum height=0.6cm, fill=green!14, font=\small},
  msg/.style={->, thick, blue!55!black},
  carry/.style={->, densely dashed, black!55},
  lab/.style={font=\scriptsize}
]
\def\dx{2.5}
\node[lab] (in1) at (0,1.0)      {$a,\ f_1$};
\node[lab] (in2) at (\dx,1.0)    {$f_2$};
\node[lab] (in3) at (2*\dx,1.0)  {$f_3$};
\node[lab] (in4) at (3*\dx,1.0)  {$f_4$};
\foreach \i in {1,2,3,4}{
  \node[player] (a\i) at ({(\i-1)*\dx},0)    {$P_\i$};
  \node[player] (b\i) at ({(\i-1)*\dx},-2.2) {$P_\i$};
  \draw[carry] (in\i) -- (a\i);
  \draw[carry] (a\i) -- (b\i);
}
\node[lab, anchor=east] at (-1.0,0)    {round $\ell=1$};
\node[lab, anchor=east] at (-1.0,-2.2) {round $\ell=2$};
\foreach \i/\j in {1/2,1/3,1/4,2/3,2/4,3/4}{
  \draw[msg] (a\i.south) to[out=-70,in=110] (b\j.north);
}
\node[outnode] (out) at (3*\dx,-4.0) {answer};
\foreach \i in {1,2,3}{
  \draw[msg] (b\i.south) to[out=-60,in=180] (out.west);
}
\draw[carry] (b4) -- (out);
\node[lab, blue!55!black, align=left, anchor=west] at ({3*\dx+0.95},-1.1)
  {each solid arrow:\\ affine block summary $(A^{\mathrm{blk}}_{\ell,i},u^{\mathrm{blk}}_{\ell,i})$,\\ $O(d^2p)$ bits};
\end{tikzpicture}
\caption{The forward communication model (Definition~\ref{def:arcomm}) for $K=4$
players and $L=2$ rounds. Player $i$ holds the explicit table of $f_i$
(player $1$ additionally holds $a$). In each synchronous round, every player
sends its affine block summary forward to all higher-indexed players. Dashed
arrows indicate locally stored information carried across rounds. After the
final round, player $K$ outputs the answer. Round $\ell$ mirrors layer $\ell$
of the SSM in the reduction of Lemma~\ref{lem:ssm-to-comm}.}
\label{fig:comm}
\end{figure}

\section{Lower Bounds for Multi-Layer SSMs}
\label{sec:lower}

\begin{theorem}[Width/precision lower bound]\label{thm:main}
Fix $N,L\in\mathbb{N}$. If an $L$-layer SSM (Definition~\ref{def:ssm}) solves the
$(L+3)$-function-composition problem under the blockwise explicit-table encoding
of Definition~\ref{def:funccomp}, with error probability at most $1/3$, then
\[
d^2p=\Omega\!\left(\frac{N}{L^3}\right).
\]
\end{theorem}

The proof follows by combining Lemma~\ref{lem:ssm-to-comm} and Lemma~\ref{lem:comm-lb}.

\paragraph{Proof roadmap.}
The argument has four steps. \emph{(1) Affine summaries:} once the layer input is
fixed, each token induces an affine map on the hidden state, so the effect of a
whole input block compresses to a $(d^2{+}d)$-scalar affine map.
\emph{(2) Forward protocol:} exchanging these summaries layer by layer converts an
$L$-layer SSM into an $L$-round protocol with messages of $O(d^2p)$ bits
(Lemma~\ref{lem:ssm-to-comm}, see Figure~\ref{fig:comm}).
\emph{(3) Serialization:} grouping the players into two super-players and
serializing the $L$ synchronous rounds yields an alternating two-party protocol
with $L{+}1$ messages (Lemma~\ref{lem:comm-lb}).
\emph{(4) Pointer chasing:} the resulting protocol solves $\mathrm{PC}_K$, so the
lower bound of \citet{mao2025gadgetless} applies. Full proofs appear in
Appendix~\ref{app:communication-proofs}.

\subsection{\texorpdfstring{SSM $\Rightarrow$ autoregressive protocol}{SSM => autoregressive protocol}}

\begin{lemma}[Reduction from SSM to autoregressive communication]\label{lem:ssm-to-comm}
Suppose there is an \(L\)-layer SSM of state dimension \(d\) and precision \(p\) that solves the \(K\)-function composition problem with error probability at most \(1/3\). Then there exists an \(L\)-round protocol in the forward communication model (Definition~\ref{def:arcomm}) that solves the same problem with the same error bound and with message length
$
c = O(d^2p).
$
\end{lemma}

\begin{proof}[Proof sketch]
Once $y_{\ell-1,t}$ is fixed, each token acts on the layer-$\ell$ state as an
affine map $h\mapsto A_{\ell,t}h+u_{\ell,t}$, and affine maps compose
associatively over an interval. Player $i$ can therefore summarize the effect of
its block on layer $\ell$ by a single pair
$(A^{\mathrm{blk}}_{\ell,i},u^{\mathrm{blk}}_{\ell,i})$ of $(d^2+d)$ $p$-bit
scalars, computable from the values $\{y_{\ell-1,t}:t\in I_i\}$ it reconstructed
in the previous round. Broadcasting these summaries forward lets every player
recover its incoming state $h_{\ell,s_i-1}$ exactly and reconstruct
$\{(h_{\ell,t},y_{\ell,t}):t\in I_i\}$, which is precisely the inductive data
needed for the next round. After round $L$, player $K$ holds the SSM output.
Each message costs $(d^2+d)p=O(d^2p)$ bits. The complete proof appears in
Appendix~\ref{app:communication-proofs}.
\end{proof}

\subsection{Lower bound for autoregressive protocols}

\begin{definition}[Two-party pointer chasing]\label{def:pc}
For \(k\ge 1\), define the \(k\)-step pointer-chasing function
$\mathrm{PC}_k : [N]^N \times [N]^N \to \{0,1\}$
as follows. Given \(f_A,f_B:[N]\to[N]\), define recursively
\[
\operatorname{pt}_0(f_A,f_B) := 1,
\qquad
\operatorname{pt}_r(f_A,f_B) :=
\begin{cases}
f_A(\operatorname{pt}_{r-1}(f_A,f_B)), & r \text{ odd},\\[2mm]
f_B(\operatorname{pt}_{r-1}(f_A,f_B)), & r \text{ even},
\end{cases}
\quad r=1,\dots,k.
\]
The output is
\[
\mathrm{PC}_k(f_A,f_B) := \operatorname{pt}_k(f_A,f_B) \bmod 2.
\]
\end{definition}

\begin{theorem}[Pointer chasing lower bound {\cite[Corollary~3]{mao2025gadgetless}}]\label{thm:pc-lb}
Every $(K-1)$-round randomized protocol for $\mathrm{PC}_K$ with error probability at most $1/3$ has total communication
\[
\Omega\!\left(\frac{N}{K}+K\right).
\]
\end{theorem}

\begin{lemma}\label{lem:comm-lb}
Let \(L,K\in\mathbb{N}\) satisfy
$
K-L \text{ is odd} \text{ and } K\ge L+3.
$
If an \(L\)-round autoregressive communication protocol solves the \(K\)-function composition problem with error probability at most \(1/3\) using messages of length at most \(c\), then
\[
c=\Omega\!\left(\frac{N}{(L+1)K^2}\right).
\]
\end{lemma}

\begin{proof}[Proof sketch]
Setting $g_i:=f_A$ for odd $i$ and $g_i:=f_B$ for even $i$ turns a
$\mathrm{PC}_K$ instance into a $K$-function-composition instance. Group the
odd-index players into Alice and the even-index players into Bob. A synchronous
round then consists of one Alice part $a_\ell$ and one Bob part $b_\ell$, each
of length at most $\lceil K/2\rceil c$. These $L$ synchronous rounds serialize
into $L+1$ alternating messages $a_1;\,(b_1,b_2);\,(a_2,a_3);\dots$ of length at
most $Kc$ each, and the parity condition on $K-L$ ensures that the last speaker
simulates player $K$. Applying Theorem~\ref{thm:pc-lb} to the resulting
$(K-1)$-round protocol gives $(L+1)Kc=\Omega(N/K+K)$, i.e.
$c=\Omega\bigl(N/((L+1)K^2)\bigr)$. The complete proof appears in
Appendix~\ref{app:communication-proofs}.
\end{proof}

\begin{proof}[Proof of Theorem~\ref{thm:main}]
Set \(K:=L+3\). Then \(K-L=3\) is odd, so Lemma~\ref{lem:comm-lb} applies. By Lemma~\ref{lem:ssm-to-comm}, the assumed SSM yields an \(L\)-round autoregressive protocol with message length
$
c=O(d^2p).
$
Hence
\[
O(d^2p)
=
c
=
\Omega\!\left(
\frac{N}{(L+1)K^2}
\right)
=
\Omega\!\left(
\frac{N}{(L+1)(L+3)^2}
\right).
\]
Since \((L+1)(L+3)^2=\Theta(L^3)\), it follows that
\[
d^2p=\Omega\!\left(\frac{N}{L^3}\right).
\]
This proves the theorem.
\end{proof}

\section{Matching Upper Bound}
\label{sec:upper}

\begin{theorem}[Composition with logarithmic memory]
\label{thm:upper-bound}
Fix $N\in\mathbb{N}$ and $L\in\mathbb{N}$ with $L\ge 2$.
There exists an $L$-layer generalized SSM (in the sense of Definition~\ref{def:ssm})
with state dimension $d=1$ and finite precision $p=\Theta(\log N)$ such that, for every input
$(a,f_1,f_2,\dots,f_{L-1})$ with $a\in[N]$ and $f_i:[N]\to[N]$, the network output at the final time
is exactly
$
f_{L-1}\bigl(f_{L-2}(\cdots f_1(a)\cdots)\bigr).
$
In particular, $dp = O(\log N)$.
\end{theorem}

\begin{proof}[Proof sketch]
The construction is a gate-and-accumulate pipeline (Figure~\ref{fig:pipeline}).
Layer $1$ loads $a$ and stores it. During the block that lists the table of
$f_i$, layer $i$ passes the entry $f_i(j)$ through exactly when the within-block
index $j$ matches its stored pointer $v_{i-1}$, and outputs $0$ otherwise. Layer
$i+1$ simply sums the block, so its state after the block equals
$v_i=f_i(v_{i-1})$ and is preserved thereafter. All intermediate quantities lie
in $\{0,\dots,N\}$, so $p=\Theta(\log N)$ bits of fixed-point precision suffice.
The complete proof appears in Appendix~\ref{app:matching-construction}.
\end{proof}

\begin{figure}[t]
\centering
\begin{tikzpicture}[
  >={Stealth[length=2mm]},
  cell/.style={draw=black!60, rounded corners=2pt, minimum height=0.66cm,
               font=\scriptsize, inner sep=3pt},
  gate/.style={cell, fill=blue!10},
  acc/.style={cell, fill=green!13},
  store/.style={cell, fill=orange!22},
  flow/.style={->, thick, black!60}
]
\def\cw{3.1}
\def\rh{1.0}
\foreach \l in {1,2,3,4}{
  \node[font=\small, anchor=east] at (-1.5,{(\l-1)*\rh}) {layer $\l$};
}
\node[store] (s0) at (0,0)            {store $v_0{=}a$};
\node[gate]  (g1) at (\cw,0)          {emit $f_1(j)$ iff $j{=}v_0$};
\node[acc]   (a1) at (\cw,\rh)        {sum block $\Rightarrow v_1$};
\node[gate]  (g2) at (2*\cw,\rh)      {emit $f_2(j)$ iff $j{=}v_1$};
\node[acc]   (a2) at (2*\cw,2*\rh)    {sum block $\Rightarrow v_2$};
\node[gate]  (g3) at (3*\cw,2*\rh)    {emit $f_3(j)$ iff $j{=}v_2$};
\node[acc]   (a3) at (3*\cw,3*\rh)    {sum block $\Rightarrow v_3$};
\draw[flow] (s0) -- (g1);
\draw[flow] (g1) -- (a1);
\draw[flow] (a1) -- (g2);
\draw[flow] (g2) -- (a2);
\draw[flow] (a2) -- (g3);
\draw[flow] (g3) -- (a3);
\node[font=\scriptsize, anchor=west] (outp) at ({3*\cw+1.75},3*\rh) {output $v_3$};
\draw[flow] (a3) -- (outp);
\node[font=\small] at (0,-0.75)      {$a$};
\node[font=\small] at (\cw,-0.75)    {table of $f_1$};
\node[font=\small] at (2*\cw,-0.75)  {table of $f_2$};
\node[font=\small] at (3*\cw,-0.75)  {table of $f_3$};
\draw[->, thick, black!55] (-1.0,-1.35) -- ({3*\cw+1.6},-1.35)
  node[below left=1pt, font=\scriptsize, black]{time (stream position)};
\end{tikzpicture}
\caption{The gate-and-accumulate pipeline of Theorem~\ref{thm:upper-bound},
shown for $K=3$ composed functions ($L=4$ layers). While the table of $f_i$
streams past, layer $i$ gates the entries against its stored pointer
$v_{i-1}$ and layer $i+1$ accumulates the unique surviving value, yielding
$v_i=f_i(v_{i-1})$. All layers use scalar state ($d=1$) and
$p=\Theta(\log N)$ bits of precision.}
\label{fig:pipeline}
\end{figure}

\subsection{Depth hierarchy}\label{sec:depth-hierarchy}
Combining the lower bound machinery of Section~\ref{sec:lower} with the upper bound of
Theorem~\ref{thm:upper-bound} yields the following depth--composition tradeoff.

\begin{corollary}[Depth--composition tradeoff for generalized SSMs]\label{cor:depth-composition-tradeoff}
Fix $K,L,N\in\mathbb{N}$ with $K\ge L+3$ and $K-L$ odd. Consider the $K$-function composition problem under the row-major stream encoding
\[
x_1:=a,\qquad
x_{\,1+(i-1)N+j}:=f_i(j)\qquad \text{for } i\in\{1,\dots,K\},\ j\in[N].
\]
Then:
\begin{enumerate}[label=(\alph*),leftmargin=2.2em]
\item If an $L$-layer generalized SSM of state dimension $d$ and precision $p$ solves this task with error probability at most $1/3$, then
\[
d^2p=\Omega\!\left(\frac{N}{(L+1)K^2}\right).
\]

\item There exists a $(K+1)$-layer generalized SSM with state dimension $d=1$ and precision $p=\Theta(\log N)$ that solves the same task exactly. In particular,
$
dp=O(\log N).
$
\end{enumerate}

In particular, setting $K:=L+3$ yields a constant-gap depth hierarchy: the $(L+3)$-function composition problem is solvable exactly by an $(L+4)$-layer generalized SSM with $dp=O(\log N)$, whereas any $L$-layer generalized SSM solving the same task with error probability at most $1/3$ must satisfy
\[
d^2p=\Omega\!\left(\frac{N}{L^3}\right).
\]
\end{corollary}

\begin{proof}
Part (a) is the direct combination of Lemma~\ref{lem:ssm-to-comm} and Lemma~\ref{lem:comm-lb}. Lemma~\ref{lem:ssm-to-comm} yields an $L$-round forward-communication protocol with message length
$
c=O(d^2p),
$
and Lemma~\ref{lem:comm-lb} then implies
\[
c=\Omega\!\left(\frac{N}{(L+1)K^2}\right).
\]
Hence
\[
d^2p=\Omega\!\left(\frac{N}{(L+1)K^2}\right).
\]

Part (b) is exactly Theorem~\ref{thm:upper-bound} applied with its layer parameter set to $K+1$. The final sentence is the specialization $K=L+3$, for which
\[
\frac{N}{(L+1)(L+3)^2}=\Theta\!\left(\frac{N}{L^3}\right).
\qedhere
\]
\end{proof}

\section{Chain-of-Thought: Post-Input versus Input-Interleaved Reasoning}
\label{sec:cot}

The SSM definition above assumes that the stream $(x_t)_{t=1}^n$ is
exogenous. Chain-of-thought (CoT) augments this stream with self-generated
tokens. We distinguish two operational regimes. In \emph{post-input
reasoning}, all self-generated tokens are produced only after the final
exogenous token. In \emph{input-interleaved reasoning}, self-generated tokens
may also be inserted between exogenous tokens (Figure~\ref{fig:cot-timing}).
For brevity, we refer to these regimes as \emph{offline CoT} and \emph{online
CoT}, respectively. Here ``online'' describes only token timing relative to
the exogenous stream. It is unrelated to online learning.

\begin{figure}[t]
\centering
\begin{tikzpicture}[
  >={Stealth[length=2mm]},
  tok/.style={draw=black!70, minimum width=0.74cm, minimum height=0.6cm,
              inner sep=1pt, font=\small},
  xin/.style={tok, fill=blue!12},
  thk/.style={tok, fill=orange!28},
  lab/.style={font=\small}
]
\def\w{0.9}
\node[lab, anchor=east] at (-0.55,0) {post-input (offline) CoT:};
\node[xin] (o1) at (0,0)        {$x_1$};
\node[xin] (o2) at (\w,0)       {$x_2$};
\node[xin] (o3) at (2*\w,0)     {$x_3$};
\node[xin] (o4) at (3*\w,0)     {$x_4$};
\node[thk] (o5) at (4*\w,0)     {$\tau_1$};
\node[thk] (o6) at (5*\w,0)     {$\tau_2$};
\node[thk] (o7) at (6*\w,0)     {$\tau_3$};
\node[lab, anchor=east] at (-0.55,-1.6) {input-interleaved (online) CoT:};
\node[xin] (i1) at (0,-1.6)       {$x_1$};
\node[thk] (i2) at (\w,-1.6)      {$\tau_1$};
\node[xin] (i3) at (2*\w,-1.6)    {$x_2$};
\node[thk] (i4) at (3*\w,-1.6)    {$\tau_2$};
\node[xin] (i5) at (4*\w,-1.6)    {$x_3$};
\node[thk] (i6) at (5*\w,-1.6)    {$\tau_3$};
\node[xin] (i7) at (6*\w,-1.6)    {$x_4$};
\node[thk] (i8) at (7*\w,-1.6)    {$\tau_4$};
\draw[->, thick, orange!75!black] (i2.south) to[out=-50,in=-130] (i3.south);
\draw[->, thick, orange!75!black] (i4.south) to[out=-50,in=-130] (i5.south);
\node[font=\scriptsize, orange!75!black] at ({3.5*\w},-2.45)
  {state-dependent feedback before the next exogenous token};
\draw[->, thick, black!55] (-0.45,-2.85) -- ({7.7*\w},-2.85)
  node[below left=1pt, font=\scriptsize, black]{time};
\end{tikzpicture}
\caption{Token timing in the two CoT regimes. Blue boxes are exogenous input
tokens, orange boxes are self-generated thought tokens. In post-input
(offline) CoT, all thought tokens appear after the final exogenous token, so
they cannot alter how the stream is absorbed. In input-interleaved (online)
CoT, thought tokens may be inserted while the stream is being read, enabling
state-dependent feedback before the next exogenous token arrives
(Theorem~\ref{thm:onlinecot-streaming}).}
\label{fig:cot-timing}
\end{figure}

\begin{definition}[Input-interleaved (online) CoT augmentation]\label{def:onlinecot}
Fix an exogenous input stream $(x_1,\dots,x_n)$.
An \emph{$L$-layer SSM with online CoT} consists of:
\begin{itemize}
\item a generalized $L$-layer SSM as in Definition~\ref{def:ssm} (with fixed finite precision $p$), and
\item a deterministic \emph{thought policy} that, after processing each exogenous token $x_i$,
produces a (possibly empty) finite sequence of thought tokens
$\tau_{i,1},\dots,\tau_{i,k_i}\in\mathbb{R}^m$, which are then fed to the SSM \emph{before}
the next exogenous token $x_{i+1}$ is revealed.
\end{itemize}
Thus the actual processed stream is
\[
x_1,\tau_{1,1},\dots,\tau_{1,k_1},\,x_2,\tau_{2,1},\dots,\tau_{2,k_2},\,\dots,\,x_n,\tau_{n,1},\dots,\tau_{n,k_n},
\]
and the output is taken at the last processed time.
We assume only that each $k_i$ is finite on every input (no uniform bound is required unless stated).

For an input $x_{1:n}$, define its total processed length by
\[
T(x_{1:n}):=n+\sum_{i=1}^{n}k_i(x_{1:n}).
\]
For worst-case resource statements on exogenous streams of length $n$, write
\[
T_n:=\max_{x_{1:n}}T(x_{1:n}),
\]
where the maximum ranges over representable exogenous streams of length $n$.
The definition requires each $k_i$ to be finite, but does not assume that
$T_n$ is polynomial in $n$.
\end{definition}

\begin{definition}[Post-input (offline) CoT augmentation]\label{def:offlinecot}
An \emph{$L$-layer SSM with offline CoT} is the special case of Definition~\ref{def:onlinecot}
where $k_i=0$ for all $i<n$, i.e.\ all thought tokens are generated only \emph{after} processing $x_n$.
Equivalently, the processed stream is
\[
x_1,x_2,\dots,x_n,\tau_{n,1},\dots,\tau_{n,k_n},
\]
and the output is taken at the last processed time.
\end{definition}

\subsection{Post-input (offline) CoT does not circumvent the communication lower-bound pipeline}

Theorem~\ref{thm:main} is the specialization $K=L+3$ of the more general lower bound obtained by combining Lemma~\ref{lem:ssm-to-comm} and Lemma~\ref{lem:comm-lb}. In particular, whenever $K\ge L+3$ and $K-L$ is odd, any $L$-layer generalized SSM solving the $K$-function composition problem must satisfy
$
d^2p=\Omega\!\left(\frac{N}{(L+1)K^2}\right),
$
and the choice $K=L+3$ yields $d^2p=\Omega(N/L^3)$.
We show that allowing offline CoT does \emph{not} change this implication.

\begin{proposition}[Offline CoT is local post-processing for the protocol reduction]\label{prop:offlinecot}
Assume an $L$-layer generalized SSM of dimension $d$ and precision $p$ solves a streaming task
(e.g.\ $K$-function composition) with error at most $1/3$, but is additionally allowed offline CoT
steps as in Definition~\ref{def:offlinecot}.
Then there exists an $L$-round protocol in the forward communication model
(Definition~\ref{def:arcomm}) that solves the same task with the same error bound and
message length $O(d^2p)$.
Consequently, every lower bound in this manuscript obtained by combining
Lemma~\ref{lem:ssm-to-comm} and Lemma~\ref{lem:comm-lb} continues to hold unchanged under offline CoT.
In particular, for $K$-function composition with $K\ge L+3$ and $K-L$ odd, one still has
$
d^2p=\Omega\!\left(\frac{N}{(L+1)K^2}\right),
$
and hence $d^2p=\Omega(N/L^3)$ when $K=L+3$.
\end{proposition}

\begin{proof}[Proof sketch]
Run the reduction of Lemma~\ref{lem:ssm-to-comm} on the exogenous stream. After
the $L$ rounds, the last player can reconstruct the entire finite-precision
stack $(h_{\ell,n})_{\ell=1}^L$, and the post-input continuation reads no
further exogenous input: all subsequent thought tokens and parameters are
deterministic functions of this state and the fixed model specification. The
last player therefore simulates the continuation locally, with no additional
communication. The complete proof appears in Appendix~\ref{app:cot-proofs}.
\end{proof}

\begin{remark}
The key point is that offline CoT occurs \emph{after} the full exogenous stream has been consumed,
so it cannot alter the information transmitted during the $L$ protocol rounds.
In contrast, online CoT may insert self-generated tokens \emph{during} streaming, which can change
how information propagates while the stream is still being read.
\end{remark}

\subsection{Input-interleaved (online) CoT is equivalent to one-pass deterministic streaming}

We now formalize the claim that online CoT renders generalized SSMs as powerful as arbitrary
deterministic one-pass streaming algorithms, at the granularity of space.

\begin{definition}[Deterministic one-pass streaming algorithm]\label{def:streaming}
Let $\mathcal{X}$ be the (finite-precision) input alphabet.
A deterministic one-pass streaming algorithm with $S$ bits of memory is specified by:
a finite memory set $\mathcal{M}$ with $|\mathcal{M}|\le 2^S$,
an initial memory state $M_0\in\mathcal{M}$, a transition function
$F:\mathcal{M}\times\mathcal{X}\to\mathcal{M}$, and an output function
$G:\mathcal{M}\to\mathcal{Y}$.
On input $(x_1,\dots,x_n)\in\mathcal{X}^n$, it iterates $M_i:=F(M_{i-1},x_i)$ and outputs $G(M_n)$.
\end{definition}

\paragraph{Operational reading of the two simulations.}
Direction (A) below is a compression statement: a streaming algorithm simply
carries the full finite-precision stack of layer states ($O(dpL)$ bits) plus a
processed-step counter, and replays each exogenous token followed by the
deterministic, finitely many thought tokens that the policy generates.
Direction (B) is a serialization statement: the SSM alternates two phases per
exogenous token. At the odd step it leaves its state untouched and re-emits the
pair (current input, current memory) as a single thought token. At the even
step, the input-dependent transition parses this token and writes the next
streaming memory $F(M,x)$ into the hidden state. One thought token per
exogenous token therefore suffices, and no adaptive stopping rule is needed.

\begin{theorem}[Input-interleaved (online) CoT $\Longleftrightarrow$ streaming]\label{thm:onlinecot-streaming}
Fix finite precision $p$, state dimension $d$, and exogenous input length $n$.
\begin{enumerate}
\item[(A)] (\emph{SSM $\Rightarrow$ streaming})
Any deterministic $L$-layer generalized SSM with online CoT (Definition~\ref{def:onlinecot})
can be simulated by a deterministic one-pass streaming algorithm whose persistent memory is
\[
O\!\left(dpL+\log T_n\right)
\]
bits on exogenous streams of length $n$, where $T_n$ is the machine's worst-case
total processed length (Definition~\ref{def:onlinecot}).
\item[(B)] (\emph{Streaming $\Rightarrow$ SSM})
Let $\mathcal{A}$ be any deterministic one-pass streaming algorithm using $S$ bits of memory
(Definition~\ref{def:streaming}).
Assume $dp\ge S$.
Then there exists a \emph{single-layer} generalized SSM with online CoT that simulates $\mathcal{A}$.
The simulation inserts exactly one thought token immediately after each exogenous token
(including $x_n$), and therefore uses exactly $2n$ generalized SSM steps.
\end{enumerate}
\end{theorem}

\begin{remark}[Time-index bookkeeping]
The additive $\log T_n$ term in part (A) accounts for storing an internal
processed-time counter. If an external clock is supplied as read-only input,
or if the relevant maps are time-homogeneous, the term can be omitted. If
$T_n=\operatorname{poly}(n)$, then $\log T_n=O(\log n)$. In particular, the
construction in part (B) has $T_n=2n$.
\end{remark}

\begin{proof}[Proof sketch]
For (A), the simulator stores the stack of layer states ($O(dpL)$ bits) and an
internal step counter (at most $T_n$, hence $O(\log T_n)$ bits), and replays
each exogenous token followed by the deterministic, finitely many thought
tokens. For (B), the SSM uses two steps per exogenous token: an odd step that
re-emits $(x_i,\text{current memory})$ as a thought token without touching the
state, and an even step whose input-dependent transition
$B_{2i}(y)=u(y)\,e_m^\top$ writes $\mathrm{Enc}(F(\mathrm{Dec}(h),x_i))$ into
the hidden state. Appendix~\ref{app:cot-proofs} gives the full constructions.
\end{proof}

\subsection{Corollary: efficient iterated function composition under online CoT}

\begin{corollary}[Online CoT SSMs solve function composition with logarithmic memory]\label{cor:cot-composition}
Consider the $K$-function composition problem where the stream presents the tables of
$f_1,\dots,f_K:[N]\to[N]$ in row-major order (after the initial token $a$), i.e.\ it enumerates the
values $f_i(1),f_i(2),\dots,f_i(N)$ for each $i$ in sequence.
There exists a deterministic one-pass streaming algorithm using $O(\log N)$ bits that outputs
$f_K(\cdots f_1(a)\cdots)$.
Consequently, by Theorem~\ref{thm:onlinecot-streaming}(B), there exists a single-layer generalized
SSM with online CoT that solves this task exactly with $dp=O(\log N)$ and with one thought token per
input token.
\end{corollary}

\begin{proof}
A one-pass streaming algorithm stores a current pointer $z\in[N]$,
initialized as $z\gets a$. At the beginning of the block encoding $f_i$, it
freezes the incoming pointer by setting $q\gets z$. While scanning
\[
f_i(1),f_i(2),\ldots,f_i(N),
\]
it maintains the within-block index $j$. When $j=q$, it stores
\[
z_{\mathrm{next}}\gets f_i(j),
\]
but leaves $q$ unchanged for the remainder of the table. At the end of the
block, it commits
\[
z\gets z_{\mathrm{next}}
\]
and proceeds to the next table. Consequently, after processing table $i$,
\[
z=f_i\bigl(f_{i-1}(\cdots f_1(a)\cdots)\bigr).
\]
The variables $z,q,z_{\mathrm{next}}$, and $j$ require
$O(\log N)$ bits in total. The algorithm is exact, and the SSM construction
follows from Theorem~\ref{thm:onlinecot-streaming}(B).
\end{proof}

\paragraph{Architectural interpretation and computational cost.}
Architecturally, input-interleaved CoT can be implemented by pausing the
external input stream after processing $x_i$, applying the ordinary SSM
update to one or more self-generated thought tokens, and then resuming with
$x_{i+1}$. When the model exposes an autoregressive feedback interface, a
generated token can simply be fed back through that interface. In the
particular construction of Corollary~\ref{cor:cot-composition}, exactly one
thought token is inserted after each exogenous token, so no adaptive stopping
mechanism is required.

For an exogenous stream of length $n$, this construction has total processed
length $T=2n$. It therefore doubles the number of recurrent updates while
preserving linear update complexity and $O(\log N)$ persistent memory. This
is an update-count statement, not a factor-two wall-clock guarantee:
autoregressive feedback may reduce the parallelism available from associative
scans, and the theorem treats its transition, readout, and thought-policy
maps abstractly. Moreover, the exact construction uses finite-precision
vector-valued thought tokens. It establishes architectural feasibility in the
generalized model, but not an efficient discrete-token realization in a
particular structured architecture such as S4 or Mamba.

\subsection{Offline--online separation for function composition}
\label{sec:cot-separation}

We can now state the separation between offline and online CoT on
the function composition benchmark.

\begin{corollary}[Offline--online CoT separation for function composition]\label{cor:offline-online-separation}
Fix $L,K\in\mathbb{N}$ with $K\ge L+3$ and $K-L$ odd. Consider the $K$-function composition problem under the row-major stream encoding
\[
x_1:=a,\qquad
x_{\,1+(i-1)N+j}:=f_i(j)\qquad \text{for } i\in\{1,\dots,K\},\ j\in[N].
\]
Then:
\begin{enumerate}[label=(\alph*),leftmargin=2.2em]
\item If an $L$-layer generalized SSM with offline CoT (Definition~\ref{def:offlinecot}) solves this task with error probability at most $1/3$, then
\[
d^2p=\Omega\!\left(\frac{N}{(L+1)K^2}\right).
\]

\item There exists an $L$-layer generalized SSM with online CoT (Definition~\ref{def:onlinecot}) that solves the same task exactly with state dimension $d=1$ and precision $p=\Theta(\log N)$. Equivalently,
\[
dp=O(\log N).
\]
Moreover, the construction may be chosen to use only one thought token per input token.
\end{enumerate}

In particular, setting $K:=L+3$ yields
\[
d^2p=\Omega\!\left(\frac{N}{L^3}\right)
\]
for offline CoT, whereas online CoT admits an exact construction with
\[
dp=O(\log N).
\]
\end{corollary}

\begin{proof}
Part (a) follows from Proposition~\ref{prop:offlinecot} together with Lemma~\ref{lem:comm-lb}, since the row-major encoding above is a valid fixed blockwise encoding for Definition~\ref{def:funccomp}.

For part (b), Corollary~\ref{cor:cot-composition} gives a single-layer generalized SSM with online CoT that solves the $K$-function composition problem exactly with $dp=O(\log N)$ and one thought token per input token. This is also an $L$-layer construction after padding with $L-1$ dummy layers that simply pass their inputs through unchanged.
\end{proof}

\section{Width versus Precision in Finite-Precision State-Space Machines}
\label{sec:width-precision}

Now, we answer the question whether, in the generalized finite-precision SSM, a width-$w$, precision-$p$ machine can always be replaced by a width-$1$, precision-$pw$ machine, or conversely. The answer depends crucially on the computational model. In the base affine-state model, the product $pw$ is \emph{not} a complete invariant: already in the one-layer case there is no universal exact simulation in either direction. By contrast, once online chain-of-thought (online CoT) is allowed, the correct invariant is total persistent memory, so a deterministic width-$w$, precision-$p$, $L$-layer machine is simulable by a width-$1$ machine of precision $O(Lpw+\log T_n)$, where $T_n$ is the total processed length of the simulated machine (Definition~\ref{def:onlinecot}), and vice versa through the same streaming-memory intermediary. So, the statement ``width can be traded for precision'' is false in the base model and true only in the stronger online-CoT model, with the correct budget $Lpw$ rather than merely $pw$ for general $L$.

\begin{definition}[Exact step-preserving simulation]
Fix two classes of deterministic machines, both driven by the same external input stream. We say that class $\mathcal{C}'$ \emph{exactly step-preservingly simulates} class $\mathcal{C}$ if, for every machine $M \in \mathcal{C}$, there exists a machine $M' \in \mathcal{C}'$ such that for every finite input stream the output produced by $M'$ at each external time step equals the output produced by $M$ at that same step. No extra external steps and no self-generated tokens are allowed.
\end{definition}

This is the natural interpretation of a width/precision tradeoff in the base model. Under this interpretation, the answer is negative in both directions.

\subsection{\texorpdfstring
  {Negative result I: width $w$ cannot, in general, be collapsed to width $1$ with precision $pw$}
  {Negative result I: width w cannot, in general, be collapsed to width 1 with precision pw}}
\begin{theorem}[No universal width-to-precision collapse in the base model]\label{thm:no-collapse}
Fix $p \ge 1$ and $w \ge 2$. In the ring model over $R_p$, the class of width-$1$, precision-$pw$ one-layer affine-state machines does \emph{not} exactly step-preservingly simulate the class of width-$w$, precision-$p$ one-layer affine-state machines.
\end{theorem}

\begin{proof}[Proof sketch]
A universal width-$w$ machine can load an arbitrary state $x\in R_p^w$, apply
an arbitrary affine map $T\in\Aff(R_p^w)$ delivered as a token, and reveal the
result on a common $\mathtt{read}$ token. Exact step-preserving simulation by a
scalar machine forces an injective encoding of states and an injective
assignment $T\mapsto F_T$ of scalar affine transitions. But
$|\Aff(R_p^w)|=2^{p(w^2+w)}$ exceeds $|\Aff(R_{pw})|=2^{2pw}$ whenever
$w\ge 2$, a contradiction. The complete proof appears in
Appendix~\ref{app:width-precision-proofs}.
\end{proof}

\begin{remark}
The proof is class-level and structural. It does not depend on any specific task lower bound. It shows directly that a one-dimensional affine state update simply has too few distinct transition maps to represent all width-$w$ affine transitions, even when the scalar has the same total number $pw$ of stored bits.
\end{remark}

\subsection{Negative result II: the reverse collapse also fails in general}

The previous theorem shows that width cannot, in general, be compressed into precision. One might hope that the reverse direction could still hold: perhaps any width-$1$, precision-$pw$ machine can be represented by a width-$w$, precision-$p$ machine. This is also false in general.

It suffices to exhibit one counterexample. We do so already for $(p,w)=(1,3)$.

\begin{theorem}[No universal precision-to-width collapse in the base model]\label{thm:no-reverse-collapse}
The class of width-$3$, precision-$1$ one-layer affine-state machines over $\F_2^3$ does \emph{not} exactly step-preservingly simulate the class of width-$1$, precision-$3$ one-layer affine-state machines over $\Z/8\Z$.
\end{theorem}

\begin{proof}[Proof sketch]
A width-$1$, precision-$3$ machine can load $s\in\Z/8\Z$ and then repeatedly
increment modulo $8$, outputting the state after each step. Any exact
step-preserving simulator over $\F_2^3$ must realize the $\mathtt{inc}$ token
by a single affine permutation of $\F_2^3$ of order $8$. A case analysis of
the possible orders of matrices in $\operatorname{GL}(3,2)$
(Lemma~\ref{lem:no-order-8}) shows that no affine permutation of $\F_2^3$ has
order $8$, a contradiction. The complete proof appears in
Appendix~\ref{app:width-precision-proofs}.
\end{proof}

\begin{remark}
Theorem~\ref{thm:no-reverse-collapse} is a counterexample, not a complete classification. In small, low-dimensional cases, accidental equivalences can occur. The point is that there is no blanket theorem asserting that width-$1$, precision-$pw$ always collapses to width-$w$, precision-$p$ in the base model.
\end{remark}

\subsection{What changes with online CoT}

We also study a stronger model in which, between external input tokens, the machine may insert finitely many self-generated thought tokens. In that online-CoT regime, the correct invariant is no longer the pair $(w,p)$ but the total amount of persistent memory.

\begin{proposition}[Online CoT collapses width and precision to total memory]\label{prop:cot-memory}
Let \(n\) be the exogenous input length, and for the machine being simulated let
\(T_n\) denote its worst-case total processed length (Definition~\ref{def:onlinecot}).
\begin{enumerate}[label=(\alph*),leftmargin=2.2em]
\item Any deterministic \(L\)-layer generalized SSM with online CoT, width \(w\), and precision \(p\), can be simulated by a deterministic single-layer generalized SSM with online CoT, width \(1\), and precision
\[
O(Lwp+\log T_n).
\]
Moreover, the simulation can be chosen to use one additional internal step per exogenous token.

\item There exists an absolute constant \(C>0\) such that any deterministic single-layer generalized SSM with online CoT, width \(1\), and precision \(q\), can be simulated by a deterministic single-layer generalized SSM with online CoT, width \(w\), and precision \(p\), provided
\[
wp \ge C(q+\log T_n).
\]
\end{enumerate}
Hence, in the deterministic online-CoT model, width and precision are interchangeable up to the total persistent memory budget.
\end{proposition}

\begin{proof}[Proof sketch]
Both directions compose the two simulations of
Theorem~\ref{thm:onlinecot-streaming}: first compress the given machine into a
deterministic one-pass streaming algorithm (part (A)), then re-express that
algorithm as a single-layer online-CoT SSM with the target width and precision
(part (B)). The complete proof appears in
Appendix~\ref{app:width-precision-proofs}.
\end{proof}

\begin{corollary}[Single-layer online-CoT width/precision tradeoff]
For deterministic single-layer online-CoT machines, width $w$ and precision $p$ are interchangeable up to the product $wp$ (and the harmless $O(\log T_n)$ bookkeeping term if the processed-time index is stored explicitly). In particular, in the online-CoT model a width-$w$, precision-$p$ machine and a width-$1$, precision-$pw+O(\log T_n)$ machine have the same exact computational power.
\end{corollary}

\section{Discussion}\label{sec:discussion}

We organize the discussion around the main themes of the paper: the
depth--composition tradeoff, the role of CoT, and the
width--precision landscape. We then highlight open problems.

\begin{center}
\fbox{%
\begin{minipage}{0.94\linewidth}
\small
\textbf{Practitioner-facing takeaways and scope.}

\begin{itemize}
\item
\textbf{Depth and composition.}
The explicit-table function-composition family exhibits a quantitative
depth--memory--precision barrier. This is not a universal lower bound for all
compositional tasks. A direct empirical test is to vary composition length
relative to SSM depth.

\item
\textbf{Post-input reasoning.}
Reasoning performed only after the complete input cannot circumvent the
communication bottleneck used by our lower-bound pipeline. This does not mean
that post-input computation is useless for every task.

\item
\textbf{Input-interleaved reasoning.}
In the deterministic generalized model, it provides the persistent-memory
power of general one-pass streaming algorithms.
Corollary~\ref{cor:cot-composition} uses one thought token per exogenous
token. Experiments should report internal-token, latency, and throughput
overhead. This does not establish learnability, preservation of scan
parallelism, or factor-two wall-clock overhead.

\item
\textbf{Width versus precision.}
Equal bit budgets do not provide a universal exact conversion between width
and precision. They should therefore be treated as separate experimental
ablation axes. The theory does not rank their practical performance.
\end{itemize}
\end{minipage}%
}
\end{center}

\paragraph{Interpreting the lower bound.}
Theorem~\ref{thm:main} is the clean specialization \(K=L+3\) of our more general
communication lower bound. It states that an \(L\)-layer SSM solving the
\((L+3)\)-function composition problem must satisfy
$
d^2 p = \Omega(N/L^3).
$
For the practically relevant regime of logarithmic precision
\(p = \Theta(\log N)\) (i.e., each scalar is represented by \(O(\log N)\) bits, as
is standard in fixed-point or floating-point implementations), this becomes
$
d^2 = \Omega\!\left(\frac{N}{L^3 \log N}\right),
$
requiring the state dimension to grow polynomially in \(N\). This is already a
strong statement: even to compose only three more functions than the layer count,
an SSM cannot solve function composition over a large domain using a compact state
unless its precision grows proportionally.

More generally, Corollary~\ref{cor:depth-composition-tradeoff} shows that whenever
\(K\ge L+3\) and \(K-L\) is odd, any \(L\)-layer SSM solving \(K\)-function
composition must satisfy
$
d^2p=\Omega\!\left(\frac{N}{(L+1)K^2}\right).
$
In contrast, Theorem~\ref{thm:upper-bound} shows that \(K\)-fold composition can be
solved exactly by a \((K+1)\)-layer SSM with \(d=1\) and
\(p=\Theta(\log N)\). In particular, setting \(K=L+3\), the same task is solvable by
an \((L+4)\)-layer SSM with \(dp=O(\log N)\). The gap between \(L\) and \(L+4\)
layers is thus a genuine constant-gap depth barrier.

\paragraph{Quadratic gap: $d^2 p$ versus $dp$.}
The lower bound scales as \(d^2 p\) while the upper bound in
Section~\ref{sec:upper} achieves \(dp = O(\log N)\). The factor \(d^2\) arises
because the affine summary communicated between blocks contains both a \(d\times d\)
matrix \(A\) and a \(d\)-dimensional vector \(b\), and it is the matrix component
that dominates the communication cost. An individual SSM layer \emph{stores} only
\(dp\) bits in its hidden state, but to \emph{transmit the effect} of a block to a
downstream player one must send the full affine map, which requires \(d^2 p\) bits.
Whether this gap can be closed, either by strengthening the lower bound to
\(dp = \Omega(\cdot)\) or by exhibiting tasks where \(d^2 p\) is truly necessary, is
an interesting open question. One might conjecture that for structured (e.g.,
diagonal or low-rank) transition matrices, the effective communication cost drops,
potentially closing the gap. See the discussion of structured parameterizations
below.

\paragraph{Implications of the CoT dichotomy.}
Post-input (offline) CoT, generating thought tokens only after the input stream has been fully
consumed, provides no benefit against the communication lower-bound pipeline
(Proposition~\ref{prop:offlinecot}). The intuition is clean: once the entire input
has been processed, all information about the input is already compressed into the
finite-precision layer states \((h_{\ell,n})_{\ell=1}^L\). Any post-hoc computation
can only manipulate this fixed, finite summary.

Input-interleaved (online) CoT, by contrast, fundamentally restructures the computation. By inserting
thought tokens \emph{during} the input stream, it allows the model to serialize its
multi-dimensional state into a scalar channel, effectively converting a multi-layer
SSM into a universal one-pass streaming algorithm
(Theorem~\ref{thm:onlinecot-streaming}). For function composition, this collapses the
resource requirement from the offline/base lower bound
$
d^2p=\Omega(N/L^3)
$
in the specialization \(K=L+3\) to an exact online-CoT construction with
$
dp=O(\log N)
$
using a single layer (Corollary~\ref{cor:cot-composition}). This has two practical
implications. First, it means that the expressiveness ceiling for SSMs with online
CoT is set by one-pass streaming lower bounds (e.g., \(\Omega(\log N)\) bits for
function composition), which are often much milder than the multi-layer lower bounds
we prove for the base model. Second, it suggests that for SSM-based language models,
the \emph{timing} of intermediate reasoning steps relative to the input may be more
important than their mere existence, a consideration that current CoT prompting
strategies~\citep{wei2022chain} for autoregressive models do not explicitly address.

We note that the online CoT model, while theoretically clean, requires a mechanism
for the model to decide when and how many thought tokens to generate. In practice,
this could be implemented via a ``pause'' or ``thinking'' token
mechanism~\citep{goyal2024think}, where the model emits special tokens that do not
contribute to the output but allow internal state manipulation. Our results provide
theoretical motivation for such mechanisms in SSM-based architectures. We emphasize,
however, that Theorem~\ref{thm:onlinecot-streaming} is an expressiveness and memory
statement: the construction of Corollary~\ref{cor:cot-composition} doubles the
number of recurrent updates and may sacrifice associative-scan parallelism, so
evaluating the accuracy--cost tradeoff of pause/thinking-token mechanisms remains an
empirical direction rather than an efficiency guarantee.

\paragraph{Width versus precision: interpretation.}
Theorems~\ref{thm:no-collapse} and~\ref{thm:no-reverse-collapse} show that, in the
base affine-state model, there is no universal exact, step-preserving conversion
between width and numerical precision, even when the product $wp$ is held fixed.
This is an algebraic non-interchangeability result. It does not establish that one
resource is empirically superior to the other. In particular, the theory does not
rank width and precision by predictive accuracy, learnability, numerical
robustness, throughput, energy use, or hardware cost. The appropriate practical
takeaway is therefore to treat state dimension and numerical precision as separate
ablation axes and report both task performance and systems cost.

However, Proposition~\ref{prop:cot-memory} shows that this non-interchangeability
vanishes under online CoT, where the sole relevant quantity is the total persistent
memory \(Lwp\), up to the harmless \(O(\log T_n)\) bookkeeping term when the
processed-time index is stored explicitly. This cleanly outlines the boundary: the
algebraic structure of matrix transitions matters in the base model but is
neutralized once the model can serialize its state through thought tokens.

\paragraph{Structured parameterizations.}
Many practical SSMs use structured transition matrices, like
diagonal~\citep{gu2022s4,smith2023s5,orvieto2023resurrecting}, block-diagonal, or
low-rank, rather than dense \(d \times d\) matrices. In such cases, the affine
summary of a block may be representable in fewer than \(d^2 p\) bits
(e.g., \(dp\) bits for a diagonal \(A\)). Our lower bound machinery
(Lemma~\ref{lem:ssm-to-comm}) automatically adapts: the communication cost per
message equals the number of bits needed to specify the block's affine summary, which
for diagonal models would be \(O(dp)\) rather than \(O(d^2 p)\). This could lead to
tighter bounds. Conversely, the upper bound construction in
Section~\ref{sec:upper} already uses \(d=1\) (trivially diagonal), so the diagonal
restriction does not weaken the achievable results. Fully characterizing the
depth--composition tradeoff for specific parameterization families (diagonal, shift,
companion, etc.) is a natural direction for future work.

\paragraph{Randomized models.}
Our lower bound framework accommodates randomized SSMs (via the randomized pointer
chasing lower bound of Theorem~\ref{thm:pc-lb}), but the upper bound constructions
are deterministic. It remains open whether randomization can provide additional power
in the base SSM model, for instance, by using random projections to compress the
affine summary more efficiently. In the streaming world, randomization is known to
yield exponential savings for certain problems (e.g., frequency
estimation~\citep{alon1999space}), and it would be interesting to determine whether
analogous separations exist for SSMs.

\paragraph{Learning versus expressiveness.}
Our results are purely about \emph{expressiveness}: we characterize which functions
can be computed by SSMs of given dimensions, not whether gradient-based training can
find the right parameters. The gap between expressiveness and learnability is well
documented for transformers~\citep{abbe2023generalization,barak2022hidden} and is
likely to be equally significant for SSMs. For instance, the construction in
Theorem~\ref{thm:upper-bound} uses a carefully designed readout function
\(\mathrm{out}_{\ell,t}\) that performs exact index matching, a function that may be
difficult to learn from data. Understanding the learnability of compositional tasks by
SSMs, and whether gradient descent on standard parameterizations can discover the
constructions we exhibit, is also a direction for future work.

\paragraph{Limitations and broader impact.}
Our results are worst-case expressiveness statements for explicit
table-encoded problems and generalized finite-precision SSMs. They do not
imply that every real-world sequence task exhibits the same lower bound.
Likewise, the positive constructions use hand-designed generalized
transition and readout maps, and the input-interleaved construction uses
vector-valued latent tokens. The results therefore do not establish
learnability by gradient descent, realization in a particular S4- or
Mamba-style parameterization, empirical accuracy, or wall-clock and hardware
efficiency.

The work uses no human-subject data, personal data, or deployed
decision-making system. Its broader impact is indirect: it may help clarify
architectural resource tradeoffs and motivate more transparent reporting of
reasoning-token, memory, and latency costs. A principal risk is
overgeneralizing worst-case formal separations into empirical claims about all
SSM tasks or implementations.

\subsection*{Open problems}

We conclude the discussion by listing concrete open problems suggested by our analysis.

\begin{enumerate}[leftmargin=2em, label=\textbf{Q\arabic*.}]

\item \textbf{Closing the \(d^2p\) versus \(dp\) gap.}
Can the lower bound for function composition be strengthened to
$
dp = \Omega(N/\mathrm{poly}(L))?
$
If not, does there exist a family of tasks \(T_N\) for which \(d^2p\) is the correct
complexity measure in the base model, in the sense that every SSM solving \(T_N\)
must satisfy
$
d^2p = \Omega(f(N)),
$
while \(T_N\) is solvable by some SSM with
$
d^2p = O(f(N)) \text{ and }
dp = o(f(N))?
$

\item \textbf{Tightening the dependence on depth and composition length.}
Our general lower bound gives
$
d^2p=\Omega\!\left(\frac{N}{(L+1)K^2}\right)
$
for \(K\ge L+3\) and \(K-L\) odd, while the upper bound shows that \(K+1\) layers
suffice to compose \(K\) functions exactly. Can the lower bound be sharpened,
either in its dependence on \(L\) and \(K\) or in the constant-gap specialization
\(K=L+3\)? For example, can the \(N/L^3\) term be improved to \(N/L^2\) or \(N/L\)?

\item \textbf{Diagonal and structured SSMs.}
What are tight depth--composition tradeoffs when \(A_{\ell,t}\) is restricted to be
diagonal, shift-structured, or low-rank?

\item \textbf{Randomized SSMs.}
Does internal randomness provably help multi-layer SSMs on compositional tasks in
the base (no-CoT) model?

\item \textbf{Learnability of compositional constructions.}
Can standard SSM training (e.g., gradient descent on Mamba-style
parameterizations) learn to solve \(K\)-function composition when \(K\) is close to
the layer count, and if so, with what sample complexity?

\item \textbf{Online CoT token complexity.}
Our online CoT constructions use one thought token per input token. Is a
sublinear number of thought tokens sufficient for function composition, or is
there a thought-token lower bound?

\end{enumerate}

\section{Conclusion}\label{sec:conclusion}

We have presented a theoretical analysis of multi-layer SSMs organized
around three topics: depth, CoT, and resource tradeoffs.
Our main result (Theorem~\ref{thm:main}) establishes that \(L\)-layer SSMs require
$
d^2 p = \Omega(N/L^3)
$
to solve the \((L+3)\)-function composition problem, via a reduction to
multi-round communication complexity through a forward communication model.
More generally, our communication argument yields the bound
$
d^2p=\Omega\!\left(\frac{N}{(L+1)K^2}\right)
$
for \(K\)-function composition whenever \(K\ge L+3\) and \(K-L\) is odd.
A complementary upper bound shows that \((K+1)\) layers suffice with
$
dp = O(\log N)
$
to compose \(K\) functions, yielding a constant-gap depth hierarchy in the
specialization \(K=L+3\).

Furthermore, we formalized the distinction between post-input (offline) and
input-interleaved (online) CoT for SSMs
and proved a sharp separation: post-input CoT does not circumvent the communication
lower-bound pipeline, while input-interleaved CoT makes multi-layer SSMs equivalent in power to
deterministic one-pass streaming algorithms. This equivalence is tight in both
directions and offers a clean characterization of the additional power conferred by
interleaving self-generated tokens with the input stream.
Finally, we showed that width and precision are \emph{not} interchangeable resources
in the base SSM model, which is a consequence of the richer algebraic structure of
matrix-valued versus scalar affine transitions, but become fully interchangeable under
online CoT, where only the total memory budget matters.

Taken together, these results provide a unified theoretical framework for
understanding how depth, finite precision, and CoT shape the
computational power of SSMs, and they suggest concrete architectural
principles: (i) depth is essential for the explicit-table composition family
studied here under the assumptions of our model, (ii) online reasoning
tokens can substitute for depth and width, and (iii) state dimension and numerical
precision play fundamentally different roles in the absence of such tokens.

\appendix

\section{Deferred Proofs for the Communication Lower Bound}
\label{app:communication-proofs}

\begin{proof}[Proof of Lemma~\ref{lem:ssm-to-comm}]
Fix an instance \((a,f_1,\dots,f_K)\) and its canonical token stream of length \(n=1+KN\) (Definition~\ref{def:funccomp}). Let
\[
I_1=[1:N+1],\qquad I_i=[2+(i-1)N:1+iN]\quad(i=2,\dots,K).
\]
These intervals partition \([n]\). Write \(I_i=[s_i:e_i]\). Player \(1\) owns the tokens encoding \(a\) and \(f_1\), while player \(i\) owns exactly the tokens encoding \(f_i\) for \(i\ge 2\).

Let \(S\) be the given \(L\)-layer SSM. Fix a layer \(\ell\in\{1,\dots,L\}\). For each time \(t\), write
$
u_{\ell,t} := B_{\ell,t}y_{\ell-1,t}\in\mathbb{R}^d,
$
so that the layer-\(\ell\) state update is
$
h_{\ell,t}=A_{\ell,t}h_{\ell,t-1}+u_{\ell,t}.
$
Thus, once \(y_{\ell-1,t}\) is regarded as fixed, each token induces an affine map
$
h\longmapsto A_{\ell,t}h+u_{\ell,t}.
$

For any interval \(J=[u:v]\subseteq [n]\), define the composed affine map
$
T_{\ell,J}(h)=A_{\ell,J}h+u_{\ell,J},
$
where
\[
A_{\ell,J}:=A_{\ell,v}A_{\ell,v-1}\cdots A_{\ell,u},
\]
and
\[
u_{\ell,J}
:=\sum_{j=u}^{v}\Bigl(\prod_{r=j+1}^{v}A_{\ell,r}\Bigr)u_{\ell,j},
\]
with the empty product interpreted as the identity. A standard induction on \(|J|\) shows that if the incoming state at time \(u-1\) is \(h_{\ell,u-1}\), then
\[
h_{\ell,v}=T_{\ell,J}(h_{\ell,u-1})=A_{\ell,J}h_{\ell,u-1}+u_{\ell,J}.
\]
Moreover, if \(J_1=[u:v]\) and \(J_2=[v+1:w]\), then
\[
T_{\ell,J_2\cup J_1}=T_{\ell,J_2}\circ T_{\ell,J_1}.
\]

For each block \(I_i=[s_i:e_i]\), define its layer-\(\ell\) summary by
\[
T_{\ell,I_i}(h)=A^{\mathrm{blk}}_{\ell,i}h+u^{\mathrm{blk}}_{\ell,i},
\qquad
A^{\mathrm{blk}}_{\ell,i}:=A_{\ell,I_i},
\quad
u^{\mathrm{blk}}_{\ell,i}:=u_{\ell,I_i}.
\]

We now describe the autoregressive protocol. Player \(i\) holds exactly the input tokens in \(I_i\).

At round \(\ell\), assume inductively that player \(i\) already knows the correct values $\{y_{\ell-1,t}: t\in I_i\}.$
From these local values, player \(i\) can instantiate all matrices \(A_{\ell,t}\), \(B_{\ell,t}\) for \(t\in I_i\), compute \(u_{\ell,t}=B_{\ell,t}y_{\ell-1,t}\), and hence compute the block summary
$\bigl(A^{\mathrm{blk}}_{\ell,i},u^{\mathrm{blk}}_{\ell,i}\bigr).$
It sends this pair to every player \(j>i\).

After receiving the summaries from players \(1,\dots,i-1\), player \(i\) forms the prefix composition
\[
P_{\ell,i-1}:=
T_{\ell,I_{i-1}}\circ \cdots \circ T_{\ell,I_1},
\qquad
P_{\ell,0}:=\operatorname{id}.
\]
Since \(h_{\ell,0}\) is fixed, player \(i\) can recover the incoming state
$
h_{\ell,s_i-1}=P_{\ell,i-1}(h_{\ell,0}).
$
It then computes sequentially, for every \(t\in I_i\),
\[
h_{\ell,t}=A_{\ell,t}h_{\ell,t-1}+u_{\ell,t},
\qquad
y_{\ell,t}=\out_{\ell,t}(h_{\ell,t},y_{\ell-1,t}).
\]
These values are stored locally and constitute the inductive data needed for round \(\ell+1\).

The induction on \(\ell\) is immediate. For \(\ell=1\), each player can compute \(y_{0,t}=\emb(x_t,t)\) on its own interval directly from the input stream. If the claim holds for layer \(\ell-1\), then the block summaries computed at round \(\ell\) agree exactly with those of the SSM, the reconstructed incoming state is correct, and hence the locally reconstructed \((h_{\ell,t},y_{\ell,t})\) are exactly the SSM values on that interval. After round \(L\), player \(K\) therefore holds exactly the SSM output, and the error probability is unchanged.

Finally, each message consists of one \(d\times d\) matrix and one \(d\)-vector, i.e.
$d^2+d$
scalars, each represented with \(p\) bits. Hence
$
c\le (d^2+d)p = O(d^2p).
$
This proves the claim.
\end{proof}

\begin{proof}[Proof of Lemma~\ref{lem:comm-lb}]
Let \(\Pi\) be such an \(L\)-round autoregressive protocol. We reduce \(\mathrm{PC}_K\) to \(K\)-function composition. Given \(f_A,f_B:[N]\to[N]\), define a composition instance by setting
\[
a:=1,
\qquad
g_i :=
\begin{cases}
f_A, & i \text{ odd},\\
f_B, & i \text{ even}.
\end{cases}
\]
Then
$
g_K(g_{K-1}(\cdots g_1(1)\cdots))
=
\operatorname{pt}_K(f_A,f_B),
$
so computing the composition value determines \(\mathrm{PC}_K(f_A,f_B)\).

Now group the \(K\) players of \(\Pi\) into two super-players: Alice simulates all odd-index players, and Bob simulates all even-index players.

For each autoregressive round \(\ell\in\{1,\dots,L\}\), let
\[
a_\ell := \text{concatenation of all round-\(\ell\) messages sent by odd players},
\]
\[
b_\ell := \text{concatenation of all round-\(\ell\) messages sent by even players}.
\]
Because the autoregressive model is synchronous, every round-\(\ell\) message depends only on the local input and on the transcript of rounds \(<\ell\). Therefore \(a_\ell\) depends only on Alice's input and on
$(a_1,b_1,\dots,a_{\ell-1},b_{\ell-1}),$
and similarly \(b_\ell\) depends only on Bob's input and the same prior transcript. Thus the pair \((a_\ell,b_\ell)\) forms an \(L\)-round simultaneous two-party protocol.

Moreover, in each round there are at most \(\lceil K/2\rceil\) odd-player messages and at most \(\lfloor K/2\rfloor\) even-player messages, so
\[
|a_\ell|\le \Bigl\lceil\frac K2\Bigr\rceil c,
\qquad
|b_\ell|\le \Bigl\lfloor\frac K2\Bigr\rfloor c.
\]

We now serialize this simultaneous protocol into an alternating one. The transcript is sent in the following order:
\[
a_1;\ (b_1,b_2);\ (a_2,a_3);\ (b_3,b_4);\ \cdots,
\]
where nonexistent terms at the end are omitted. More formally, the alternating messages are
\[
m_1:=a_1,
\]
\[
m_{2r}:=(b_{2r-1},b_{2r}),
\qquad
m_{2r+1}:=(a_{2r},a_{2r+1}),
\]
with missing components deleted when an index exceeds \(L\).

This serialization is valid because once a party knows the simultaneous transcript through round \(t-1\) and has just computed its own round-\(t\) message, it can also compute its round-\((t+1)\) message immediately after receiving the other party's round-\(t\) message. Concretely:

\begin{itemize}
\item after Alice sends \(a_1\), Bob can compute \(b_1\), and then \(b_2\);
\item after Bob sends \((b_1,b_2)\), Alice can compute \(a_2\), and then \(a_3\);
\item and so on.
\end{itemize}

Hence the simultaneous transcript is fully serialized in \(L+1\) alternating messages. Each such message has length at most
\[
|a_{\ell}|+|a_{\ell+1}| \le Kc
\quad\text{or}\quad
|b_{\ell}|+|b_{\ell+1}| \le Kc,
\]
so the total communication is \(O((L+1)Kc)\).

At the end of these \(L+1\) messages, both parties know the entire simulated simultaneous transcript. The last speaker is Alice when \(L\) is even and Bob when \(L\) is odd. Since Alice simulates odd-index players and Bob simulates even-index players, the last speaker simulates player \(K\) exactly when \(K-L\) is odd, which is our assumption. Therefore the last speaker can compute the value of player \(K\), namely
$
\operatorname{pt}_K(f_A,f_B),
$
and append the output bit
$
\operatorname{pt}_K(f_A,f_B)\bmod 2
$
to its final message. We have thus produced an Alice-first alternating protocol for \(\mathrm{PC}_K\) with at most \(L+1\) messages and total communication \(O((L+1)Kc)\). If necessary, we pad with empty dummy messages to obtain exactly \(K-1\) rounds. This does not change the asymptotic communication.

Since \(K\ge L+3\), we have \(L+1\le K-1\). Theorem~\ref{thm:pc-lb} therefore applies and yields
\[
(L+1)Kc = \Omega\!\left(\frac{N}{K}+K\right).
\]
Rearranging,
\[
c
=
\Omega\!\left(
\frac{N/K+K}{(L+1)K}
\right)
=
\Omega\!\left(
\frac{N}{(L+1)K^2}+\frac{1}{L+1}
\right)
=
\Omega\!\left(
\frac{N}{(L+1)K^2}
\right).
\]
This proves the lemma.
\end{proof}

\section{Deferred Proof of the Matching Construction}
\label{app:matching-construction}

\begin{proof}[Proof of Theorem~\ref{thm:upper-bound}]
As already said, we work in the generalized multi-layer SSM model where, for each layer $\ell\in\{1,\dots,L\}$ and time
$t\ge 1$, the hidden state and output obey
\[
h_{\ell,t} = A_{\ell,t} h_{\ell,t-1} + B_{\ell,t} y_{\ell-1,t},\qquad
y_{\ell,t} = \mathrm{out}_{\ell,t}(h_{\ell,t},y_{\ell-1,t}),
\]
and $y_{0,t}=\mathrm{emb}(x_t,t)$ is the embedded input token stream. We will choose $m=1$ and $d=1$,
so all quantities above are scalars.

\paragraph{Encoding of the input stream.}
We fix the following (canonical) stream encoding of the instance $(a,f_1,\dots,f_{L-1})$.
Let $K:=L-1$ and define the stream length
\[
n := 1 + KN = 1 + (L-1)N.
\]
Define tokens $x_1,x_2,\dots,x_n$ by
\[
x_1 := a,\qquad
x_{\,1+(i-1)N+j} := f_i(j)\quad\text{for }i\in\{1,\dots,K\},\; j\in\{1,\dots,N\}.
\]
Thus, after the first token, the stream lists the truth tables of $f_1,f_2,\dots,f_{K}$ consecutively in
row-major order. Define the embedding to be the identity
$
\mathrm{emb}(x_t,t) := x_t,
$
so that $y_{0,t}=x_t$.

\paragraph{Precision choice.}
Choose a fixed-point (or integer) encoding with $p$ bits per scalar large enough to represent every
integer in $\{0,1,\dots,N\}$ exactly. For concreteness, it suffices to take
$
p \;\ge\; \lceil \log_2(N+1)\rceil + 1.
$
All quantities we compute below will lie in $\{0,1,\dots,N\}$ and hence are exactly representable at this
precision.

\paragraph{High-level idea.}
Write
\[
v_0 := a,\qquad v_i := f_i(v_{i-1})\quad(i=1,\dots,K).
\]
We will maintain the invariant that, after the block encoding $f_i$ has been fully read, layer $i+1$
stores $v_i$ in its hidden state. The computation of $v_i$ is done in a two-layer pipeline:
layer $i$ \emph{gates} the table entries of $f_i$ so that exactly one nonzero value passes to layer $i+1$,
and layer $i+1$ \emph{accumulates} these values over the block.

\paragraph{Definition of the SSM parameters.}
All layers have scalar state ($d=1$) and scalar outputs ($m=1$). Set initial states
$h_{\ell,0}=0$ for all $\ell\in\{1,\dots,L\}$.

Define time indices for each function block:
for $i\in\{1,\dots,K\}$, let
\[
s_i := 2+(i-1)N,\qquad e_i := 1+iN,
\]
so the $i$th function block occupies times $t\in[s_i:e_i]$ and has length $N$.
For each such $t\in[s_i:e_i]$, define the within-block index
\[
j_i(t) := t-s_i+1 \in \{1,\dots,N\},
\]
so that $x_t = f_i(j_i(t))$ on that block.

\smallskip
\noindent\textbf{Layer 1 (store $a$).}
Set
\[
A_{1,1}=0,\quad B_{1,1}=1,\qquad
A_{1,t}=1,\quad B_{1,t}=0\quad\text{for all }t\ge 2.
\]
Thus $h_{1,1}=y_{0,1}=a$ and thereafter $h_{1,t}=h_{1,t-1}=a$ for all $t\ge 2$.

\smallskip
\noindent\textbf{Layers $\ell\in\{2,\dots,L\}$ (accumulate exactly one gated value).}
For $\ell\ge 2$, define $i:=\ell-1$ (so $i\in\{1,\dots,K\}$ when $\ell\le L$). Set
\[
A_{\ell,t} := 1\quad\text{for all }t\ge 1,
\]
and
\[
B_{\ell,t} :=
\begin{cases}
1, & t\in[s_i:e_i],\\
0, & \text{otherwise}.
\end{cases}
\]
Hence layer $\ell=i+1$ performs the update
$h_{\ell,t}=h_{\ell,t-1}+y_{\ell-1,t}$ during the $i$th block and remains constant outside that block.

\smallskip
\noindent\textbf{Readout maps.}
For each $\ell\in\{1,\dots,K\}$ and time $t\ge 1$, define $\mathrm{out}_{\ell,t}$ by
\[
\mathrm{out}_{\ell,t}(h,y) :=
\begin{cases}
y, & t\notin[s_\ell:e_\ell],\\[2mm]
y, & t\in[s_\ell:e_\ell]\ \text{and}\ h=j_\ell(t),\\[2mm]
0, & t\in[s_\ell:e_\ell]\ \text{and}\ h\neq j_\ell(t).
\end{cases}
\]
In words: layer $\ell$ passes its input through unchanged at all times except during the block encoding
$f_\ell$, where it outputs the current table entry iff its stored pointer equals the current index.
Finally, for the last layer $L$, define
\[
\mathrm{out}_{L,t}(h,y) := h\qquad\text{for all }t\ge 1.
\]
Therefore the network output at time $t$ is $y_t=y_{L,t}=h_{L,t}$.

\paragraph{Correctness.}
We prove by induction on $i\in\{0,1,\dots,K\}$ that
\begin{equation}\label{eq:invariant}
h_{i+1,t} = v_i\quad\text{for all }t\ge e_i,
\end{equation}
where we interpret $e_0:=1$ (so that the claim for $i=0$ states $h_{1,t}=v_0=a$ for all $t\ge 1$).

\smallskip
\noindent\emph{Base case $i=0$.}
By construction, $h_{1,1}=a=v_0$ and for $t\ge 2$ we have $h_{1,t}=h_{1,t-1}$, so
$h_{1,t}=v_0$ for all $t\ge 1=e_0$.

\smallskip
\noindent\emph{Inductive step.}
Fix $i\in\{1,\dots,K\}$ and assume \eqref{eq:invariant} holds for $i-1$, i.e.,
$h_{i,t}=v_{i-1}$ for all $t\ge e_{i-1}$.
In particular, throughout the entire block $t\in[s_i:e_i]$ we have $t\ge e_{i-1}+1$, hence
\[
h_{i,t} = v_{i-1}\qquad\text{for all }t\in[s_i:e_i].
\]

Next observe that for all layers $\ell<i$ and times $t\in[s_i:e_i]$, we have $t\notin[s_\ell:e_\ell]$
because the blocks are consecutive and strictly increasing in $i$.
Therefore, for such $t$, the readouts satisfy $y_{\ell,t}=y_{\ell-1,t}$ for all $\ell<i$, from where
\[
y_{i-1,t} = y_{0,t} = x_t = f_i(j_i(t))\qquad\text{for all }t\in[s_i:e_i].
\]
Applying the definition of $\mathrm{out}_{i,t}$ on the $i$th block yields
\[
y_{i,t} =
\begin{cases}
f_i(j_i(t)), & j_i(t)=h_{i,t}=v_{i-1},\\
0, & \text{otherwise},
\end{cases}
\qquad\text{for }t\in[s_i:e_i].
\]
Since $j_i(t)$ ranges over $\{1,\dots,N\}$ exactly once as $t$ ranges over $[s_i:e_i]$, there is a unique
time $t^\star\in[s_i:e_i]$ for which $j_i(t^\star)=v_{i-1}$, and at that time
$y_{i,t^\star}=f_i(v_{i-1})=v_i$, while $y_{i,t}=0$ for all $t\neq t^\star$ in the block.

Now consider layer $i+1$. By construction, $B_{i+1,t}=0$ for all $t<s_i$, so
$h_{i+1,s_i-1}=h_{i+1,0}=0$. For $t\in[s_i:e_i]$ we have $B_{i+1,t}=1$ and $A_{i+1,t}=1$, hence
\[
h_{i+1,t} = h_{i+1,t-1} + y_{i,t}\qquad(t\in[s_i:e_i]).
\]
Unrolling the recurrence to the end of the block gives
\[
h_{i+1,e_i} = \sum_{t=s_i}^{e_i} y_{i,t} = v_i.
\]
Finally, for $t>e_i$ we have $B_{i+1,t}=0$ and $A_{i+1,t}=1$, so $h_{i+1,t}=h_{i+1,t-1}$ and thus
$h_{i+1,t}=v_i$ for all $t\ge e_i$. This proves \eqref{eq:invariant} for $i$.

\smallskip
\noindent\emph{Conclusion.}
Taking $i=K=L-1$, we have $e_K=n$, so the invariant yields
\[
h_{L,n} = v_{L-1} = f_{L-1}\bigl(f_{L-2}(\cdots f_1(a)\cdots)\bigr).
\]
Since $\mathrm{out}_{L,n}(h,y)=h$, the network output at time $n$ is $y_{L,n}=h_{L,n}$, which is the
desired value.

\paragraph{Resource bound.}
We used state dimension $d=1$. Choosing $p\ge \lceil \log_2(N+1)\rceil+1$ ensures all intermediate
values lie in $\{0,1,\dots,N\}$ and are exactly representable. Hence $dp = O(\log N)$.
\end{proof}

\section{Deferred Proofs for the CoT Results}
\label{app:cot-proofs}

\begin{proof}[Proof of Proposition~\ref{prop:offlinecot}]
Fix an input instance and its associated exogenous token stream $(x_t)_{t=1}^n$
(as in Definition~\ref{def:funccomp}). Consider the execution of the offline-CoT SSM:
it processes $(x_t)_{t=1}^n$ first, reaching some global internal state at time $n$
(consisting of all layer states $(h_{\ell,n})_{\ell=1}^L$ and any other finite-precision registers
implicit in the implementation), and then performs additional offline CoT time steps
using only self-generated tokens and the fixed model specification.

We construct an $L$-round protocol in the forward communication model as follows.
\begin{enumerate}
\item Use \emph{exactly} the reduction of Lemma~\ref{lem:ssm-to-comm} to simulate the SSM \emph{up to time $n$}
on the exogenous stream. This produces, for each layer $\ell$, the correct affine block summaries
and therefore allows the last player (player $K$) to reconstruct
the layer-$\ell$ state at the end of the exogenous stream, i.e.\ $h_{\ell,n}$, by composing the received
block summaries and applying the resulting affine map to the fixed initial state.
The communication per message is the same as in Lemma~\ref{lem:ssm-to-comm}, namely $O(d^2p)$ bits.

\item After the $L$ protocol rounds, no further communication is performed.
The last player now possesses (i) the full model description and (ii) the reconstructed stack of
finite-precision layer states $(h_{\ell,n})_{\ell=1}^L$. Starting from this state, it can
\emph{locally} simulate the offline CoT continuation (Definition~\ref{def:offlinecot}), because:
\begin{itemize}
\item the continuation reads no further exogenous input, and
\item all future thought tokens and SSM parameters are deterministic functions of the current
finite-precision state and the fixed model specification.
\end{itemize}
Therefore the last player can reproduce the same final output distribution that the offline-CoT SSM
would produce.
\end{enumerate}
Hence we obtain an $L$-round forward-communication protocol with message length $O(d^2p)$ and the same
error guarantee. Any subsequent lower bound argument that applies to such protocols, such as Lemma~\ref{lem:comm-lb}, applies verbatim, proving the claim.
\end{proof}

\begin{proof}[Proof of Theorem~\ref{thm:onlinecot-streaming}]
\noindent\textbf{(A) SSM $\Rightarrow$ streaming.}\quad
Fix a deterministic $L$-layer SSM with online CoT.
Because all hidden-state coordinates are $p$-bit scalars, each layer state $h_{\ell,t}\in\mathbb{R}^d$
can be stored in $O(dp)$ bits, hence the entire stack $(h_{1,t},\dots,h_{L,t})$ in $O(dpL)$ bits.

A streaming simulator proceeds token-by-token over the exogenous stream $(x_i)_{i=1}^n$.
Upon reading $x_i$, it:
\begin{enumerate}
\item computes $y_{0,t}=\mathrm{emb}(x_i,t)$ and updates the layer states sequentially using \eqref{eq:ssm} to obtain the new outputs and states at that time step;
\item computes the thought policy's next thought token(s) (if any) as deterministic functions of the
current finite-precision configuration;
\item feeds each thought token back through $\mathrm{emb}$ and applies the same state update, repeating
until the thought policy halts for this $i$ (which is guaranteed to happen after finitely many steps
by Definition~\ref{def:onlinecot}).
\end{enumerate}
The simulator carries forward only the current finite-precision stack state (and a counter tracking
how many internal steps have occurred, which is at most the total processed length $T_n$ and hence
requires $O(\log T_n)$ bits), which is $O(dpL+\log T_n)$ bits in total. This exactly reproduces the
SSM's final output, hence simulates it.

\paragraph{(B) Streaming $\Rightarrow$ SSM.}
Let $\mathcal{A}$ be a deterministic streaming algorithm with memory set $\mathcal{M}$,
$|\mathcal{M}|\le 2^S$, transition function $F$, and output $G$.
Assume $dp\ge S$.
Choose an injective encoding $\mathrm{Enc}:\mathcal{M}\to\{0,\dots,2^p-1\}^d\subset\mathbb{R}^d$.
(For example, fix a bijection between $\mathcal{M}$ and a subset of $\{0,1\}^{dp}$ and group the
$dp$ bits into $d$ blocks of $p$ bits.) Let $\mathrm{Dec}$ denote its inverse on $\mathrm{Enc}(\mathcal{M})$.

We construct a \emph{single-layer} SSM (so we drop the layer index) that processes an augmented
stream of length $2n$ consisting of exogenous tokens at odd times and one thought token at each even
time. We choose the token dimension
$
m := m_x + d + 1,
$
where $m_x$ is the (finite-precision) dimension used to represent exogenous inputs.
We assume the embedding appends a bias coordinate:
$
\mathrm{emb}(x,t):=(x,0^d,1)\in\mathbb{R}^{m}.
$

\smallskip\noindent
\emph{Odd times ($t=2i-1$): expose $(x_i,M_{i-1})$ as a thought token without changing state.}
Set
\[
A_{2i-1}=I_d,\qquad B_{2i-1}=0.
\]
Thus $h_{2i-1}=h_{2i-2}$.
Define the readout map at odd times by
\[
\mathrm{out}_{2i-1}(h,y):=\bigl(x,\,h,\,1\bigr)\in\mathbb{R}^{m},
\]
where $x$ denotes the first $m_x$ coordinates of $y=\mathrm{emb}(x_i,2i-1)$.
Hence the layer output at time $2i-1$ is
\[
y_{2i-1}=(x_i,h_{2i-2},1).
\]
We interpret $y_{2i-1}$ as the (single) thought token inserted before $x_{i+1}$.

\smallskip\noindent
\emph{Even times ($t=2i$): update the hidden state to the next streaming memory.}
At time $2i$, the current token is the thought token $y_{0,2i}=y_{2i-1}$.
Set
\[
A_{2i}=0,
\]
and define a \emph{matrix-valued} function $B_{2i}(\cdot)$ on the set of representable inputs
$y\in\mathbb{R}^{m}$ by
\[
B_{2i}(y):=u(y)\,e_m^\top,
\qquad
u(y):=\mathrm{Enc}\Bigl(F\bigl(\mathrm{Dec}(h),x\bigr)\Bigr)\in\mathbb{R}^d,
\]
where $y=(x,h,1)$ is parsed into its $m_x$-coordinate input part $x$, its $d$-coordinate state part
$h$, and its bias coordinate $1$, and $e_m$ is the $m$th standard basis vector.
Because the last coordinate of $y$ equals $1$, we get
\[
h_{2i} \;=\; A_{2i}h_{2i-1}+B_{2i}(y)\,y \;=\; u(y)\,(e_m^\top y) \;=\; u(y).
\]
Therefore
$
h_{2i}=\mathrm{Enc}\Bigl(F\bigl(\mathrm{Dec}(h_{2i-2}),x_i\bigr)\Bigr).
$
Initialize the SSM with $h_0:=\mathrm{Enc}(M_0)$.
An induction on $i$ shows $h_{2i}=\mathrm{Enc}(M_i)$ for all $i\in\{0,1,\dots,n\}$.

Finally, define the readout at the last time $2n$ to emit the streaming output:
\[
\mathrm{out}_{2n}(h,y) := \mathrm{Enc}_\mathcal{Y}\bigl(G(\mathrm{Dec}(h))\bigr),
\]
where $\mathrm{Enc}_\mathcal{Y}$ is any fixed embedding of the output alphabet $\mathcal{Y}$
into $\mathbb{R}^m$ (e.g.\ store it in the first coordinate and zero elsewhere).
Then the SSM output at time $2n$ equals $\mathcal{A}$'s output on the exogenous stream.

All quantities involved are finite-precision: the state $h_t$ always lies in the representable set
$\mathrm{Enc}(\mathcal{M})\subset\{0,\dots,2^p-1\}^d$, and $B_{2i}(y)$ has entries in $\{0,\dots,2^p-1\}$
because $u(y)$ does. The construction inserts exactly one thought token after each exogenous token
(including $x_n$), so the total processed length is exactly $2n$. This completes the construction.
\end{proof}

\section{Deferred Proofs for the Width--Precision Results}
\label{app:width-precision-proofs}

\begin{proof}[Proof of Theorem~\ref{thm:no-collapse}]
Let
$
\mathcal{A}_{w,p} := \Aff(R_p^w)=\{h\mapsto Ah+b : A\in R_p^{w\times w},\ b\in R_p^w\}
$
be the full set of affine self-maps of $R_p^w$. Consider the following width-$w$, precision-$p$ one-layer machine $U_{w,p}$.

\smallskip
\noindent
\emph{Input alphabet.} The first token is an element $x\in R_p^w$. The second token is an affine map $T=(A,b)\in \mathcal{A}_{w,p}$. The third token is a fixed symbol $\mathtt{read}$.

\smallskip
\noindent
\emph{Dynamics.}
\begin{itemize}[leftmargin=2em]
\item At time $t=1$, the machine loads the hidden state with the input vector:
$
h_1=x.
$
This is realized by choosing $A_{1,x}=0$ and $b_{1,x}=x$.
\item At time $t=2$, upon reading $T=(A,b)$, the machine applies
$
h_2 = Ah_1+b,
$
and produces no relevant output.
\item At time $t=3$, upon reading the common token $\mathtt{read}$, the machine leaves the state unchanged and outputs the current hidden state.
\end{itemize}
Thus, on input $(x,T,\mathtt{read})$, the output at time $3$ is exactly $T(x)$.

Assume, for contradiction, that there exists a width-$1$, precision-$pw$ one-layer affine-state machine $V$ that exactly step-preservingly simulates $U_{w,p}$.

Let $S:=R_{pw}$ be the scalar state space of $V$. For each $x\in R_p^w$, let $s_x\in S$ denote the state of $V$ after reading the first token $x$.

We first claim that the map
\[
E:R_p^w\to S,\qquad E(x):=s_x,
\]
is injective. Indeed, if $E(x)=E(x')$ for some $x\neq x'$, then on the common suffix $(\mathrm{id},\mathtt{read})$ the simulator would be in the same scalar state at times $2$ and $3$ for both inputs $(x,\mathrm{id},\mathtt{read})$ and $(x',\mathrm{id},\mathtt{read})$, hence would produce the same output at time $3$. But $U_{w,p}$ outputs $x$ on the first input and $x'$ on the second, contradiction. Since
$
|R_p^w| = 2^{pw}=|R_{pw}|,
$
this injective map is in fact bijective.

For each affine map $T=(A,b)\in \mathcal{A}_{w,p}$, let
\[
F_T(s)=\alpha_T s+\beta_T\qquad (\alpha_T,\beta_T\in R_{pw})
\]
be the scalar affine map implemented by $V$ at time $2$ on input token $T$.

We next claim that the assignment $T\mapsto F_T$ is injective. Suppose $F_T=F_{T'}$ for two affine maps $T\neq T'$. Since $T$ and $T'$ are distinct functions on $R_p^w$, there exists $x\in R_p^w$ with $T(x)\neq T'(x)$. The simulator, started from the first token $x$, reaches the same scalar state after time $2$ on the input $(x,T,\mathtt{read})$ as on $(x,T',\mathtt{read})$, because the scalar transition map at time $2$ is the same. At time $3$ the current input token is the common symbol $\mathtt{read}$, and the simulator's state is the same in both executions, so the produced output at time $3$ must also be the same. This contradicts exact simulation, because $U_{w,p}$ outputs $T(x)$ in the first execution and $T'(x)$ in the second.

Hence the number of affine self-maps of $R_p^w$ cannot exceed the number of scalar affine self-maps of $R_{pw}$. But
$
|\Aff(R_p^w)| = |R_p|^{w^2+w}=2^{p(w^2+w)},
$
whereas
$
|\Aff(R_{pw})| = |R_{pw}|^2 = 2^{2pw}.
$
For every $w\ge 2$ we have $w^2+w>2w$, so
$
2^{p(w^2+w)} > 2^{2pw},
$
a contradiction. Therefore no such simulator $V$ exists.
\end{proof}

\begin{proof}[Proof of Theorem~\ref{thm:no-reverse-collapse}]
Consider the following width-$1$, precision-$3$ machine $C$ over the input alphabet
$
X := \Z/8\Z\ \sqcup\ \{\mathtt{inc}\}.
$
On the first input token $s\in \Z/8\Z$, the machine loads the hidden state with $s$. On every subsequent token $\mathtt{inc}$, it updates
\[
h\longmapsto h+1 \pmod 8
\]
and outputs the new state.

Assume, for contradiction, that there exists a width-$3$, precision-$1$ one-layer affine-state simulator $D$ over the state space
$
V:=\F_2^3.
$
Let $e(s)\in V$ be the simulator state after reading the first token $s$. Since exact simulation must hold for all future suffixes of $\mathtt{inc}$ tokens, the map $s\mapsto e(s)$ is injective, because both sets have size $8$, it is bijective.

Let
\[
f(x)=Ax+b\qquad (A\in \operatorname{GL}(3,2),\ b\in \F_2^3)
\]
be the affine self-map of $V$ used by the simulator on the token $\mathtt{inc}$. Fix any $s\in \Z/8\Z$. Exact simulation implies that after $k$ successive $\mathtt{inc}$ tokens, the output must be $s+k\pmod 8$. In particular, for $k=1,2,\dots,8$ these outputs are pairwise distinct. Since the current input token is the same in each of these steps, equal simulator states would force equal outputs. Therefore the states
$
f(e(s)),\ f^2(e(s)),\ \dots,\ f^8(e(s))
$
are pairwise distinct. As $V$ has exactly $8$ elements, $f$ must be a permutation of $V$ with a single orbit of length $8$, and equivalently, $f$ must have order $8$.

We now show that no affine permutation of $\F_2^3$ has order $8$, which yields the contradiction.

\begin{lemma}\label{lem:no-order-8}
If $f(x)=Ax+b$ is an affine permutation of $\F_2^3$, then the order of $f$ is not equal to $8$.
\end{lemma}

\begin{proof}
Write $r:=\operatorname{ord}(A)$. Since
\[
f^m(x)=A^m x + \sum_{i=0}^{m-1}A^i b,
\]
we obtain
\[
f^r(x)=x+c,\qquad c:=\sum_{i=0}^{r-1}A^i b.
\]
Because $V=\F_2^3$ has characteristic $2$, every translation $x\mapsto x+c$ has order $1$ or $2$. Hence the order of $f$ divides $2r$.

It remains to understand the possible orders of $A\in\operatorname{GL}(3,2)$. The characteristic polynomial of such an $A$ has degree $3$ and nonzero constant term, so its irreducible factors over $\F_2$ are among
\[
x+1,\qquad x^2+x+1,\qquad x^3+x+1,\qquad x^3+x^2+1.
\]
The latter two cubic polynomials are primitive, hence contribute order $7$, the quadratic polynomial contributes order $3$, and the unipotent factor $(x+1)^m$ contributes $2$-power order at most $4$ in dimension $3$. Therefore, $
r\in\{1,2,3,4,7\}.
$
If $r\in\{1,2,3,7\}$, then every divisor of $2r$ belongs to
$
\{1,2,3,4,6,7,14\},
$
so in particular is not $8$.

The only remaining possibility is $r=4$. In that case $A$ is unipotent of index $3$, so we may write $A=I+N$ with $N^3=0$. Then in characteristic $2$,
\[
A^2=(I+N)^2=I+N^2,
\qquad
A^3=(I+N)^3=I+N+N^2,
\]
and therefore
\[
I+A+A^2+A^3 = I+(I+N)+(I+N^2)+(I+N+N^2)=0.
\]
Consequently,
\[
c=(I+A+A^2+A^3)b=0,
\]
so $f^4=\mathrm{id}$. Thus in the case $r=4$ the order of $f$ divides $4$, again not $8$.

Therefore no affine permutation of $\F_2^3$ has order $8$.
\end{proof}

By Lemma~\ref{lem:no-order-8}, the simulator transition on $\mathtt{inc}$ cannot have order $8$, contradicting the previous paragraph. Hence no such simulator exists.
\end{proof}

\begin{proof}[Proof of Proposition~\ref{prop:cot-memory}]
For part (a), apply Theorem~\ref{thm:onlinecot-streaming}(A) with state dimension \(d=w\). The given \(L\)-layer online-CoT SSM is simulated by a deterministic one-pass streaming algorithm using
$
O(Lwp+\log T_n)
$
bits of persistent memory. Now apply Theorem~\ref{thm:onlinecot-streaming}(B) to this streaming algorithm with target state dimension \(d=1\) and precision
$
p' = O(Lwp+\log T_n).
$
This yields a deterministic single-layer generalized SSM with online CoT, width \(1\), and precision \(p'\) that simulates the original machine. By the construction in Theorem~\ref{thm:onlinecot-streaming}(B), the simulation uses one additional internal step per exogenous token.

For part~(b), apply Theorem~\ref{thm:onlinecot-streaming}(A) to the given single-layer online-CoT SSM of width \(1\) and precision \(q\). This produces a deterministic one-pass streaming algorithm using
\[
O(q+\log T_n)
\]
bits of persistent memory. Choose \(C>0\) sufficiently large to dominate the hidden constant in this bound. If
$
wp \ge C(q+\log T_n),
$
then Theorem~\ref{thm:onlinecot-streaming}(B), applied with target state dimension \(d=w\) and precision \(p\), yields a deterministic single-layer generalized SSM with online CoT, width \(w\), and precision \(p\) that simulates the original machine. The final sentence follows by combining parts (a) and (b).
\end{proof}

\bibliography{main}

@inproceedings{gu2022s4,
  title={Efficiently Modeling Long Sequences with Structured State Spaces},
  author={Gu, Albert and Goel, Karan and Re, Christopher},
  booktitle={International Conference on Learning Representations},
  year={2022}
}

@inproceedings{
gu2023mamba,
title={Mamba: Linear-Time Sequence Modeling with Selective State Spaces},
author={Albert Gu and Tri Dao},
booktitle={First Conference on Language Modeling},
year={2024},
}

@inproceedings{dao2024mamba2,
  title={Transformers are SSMs: generalized models and efficient algorithms through structured state space duality},
  author={Dao, Tri and Gu, Albert},
  booktitle={Proceedings of the 41st International Conference on Machine Learning},
  pages={10041--10071},
  year={2024}
}

@inproceedings{merrill2024illusion,
  title={The Illusion of State in State-Space Models},
  author={Merrill, William and Petty, Jackson and Sabharwal, Ashish},
  booktitle={International Conference on Machine Learning},
  pages={35492--35506},
  year={2024},
  organization={PMLR}
}

@article{sarrof2024expressive,
  title={The expressive capacity of state space models: A formal language perspective},
  author={Sarrof, Yash and Veitsman, Yana and Hahn, Michael},
  journal={Advances in Neural Information Processing Systems},
  volume={37},
  pages={41202--41241},
  year={2024}
}

@article{cirone2024theoretical,
  title={Theoretical foundations of deep selective state-space models},
  author={Muca Cirone, Nicola and Orvieto, Antonio and Walker, Benjamin and Salvi, Cristopher and Lyons, Terry},
  journal={Advances in Neural Information Processing Systems},
  volume={37},
  pages={127226--127272},
  year={2024}
}

@inproceedings{merrill2024cot,
  title={The Expressive Power of Transformers with Chain of Thought},
  author={Merrill, William and Sabharwal, Ashish},
  booktitle={The Twelfth International Conference on Learning Representations},
  year      = {2024},
}

@inproceedings{li2024cot,
  title={Chain of thought empowers transformers to solve inherently serial problems},
  author={Li, Zhiyuan and Liu, Hong and Zhou, Denny and Ma, Tengyu},
  booktitle={The Twelfth International Conference on Learning Representations},
  year      = {2024},
}

@article{nisan1993rounds,
  title={Rounds in communication complexity revisited},
  author={Nisan, Noam and Wigderson, Avi},
  journal={SIAM Journal on Computing},
  volume={22},
  number={1},
  pages={211--219},
  year={1993},
}

@inproceedings{mao2025gadgetless,
  title={Gadgetless Lifting Beats Round Elimination: Improved Lower Bounds for Pointer Chasing},
  author={Mao, Xinyu and Yang, Guangxu and Zhang, Jiapeng},
  booktitle={16th Innovations in Theoretical Computer Science Conference (ITCS 2025)},
  year={2025},
}

@article{yehudayoff2020pointer,
  title={Pointer chasing via triangular discrimination},
  author={Yehudayoff, Amir},
  journal={Combinatorics, Probability and Computing},
  volume={29},
  number={4},
  pages={485--494},
  year={2020},
}

@article{ponzio1999pointer,
  title={The communication complexity of pointer chasing. Applications of entropy and sampling},
  author={Stephen Ponzio and Jaikumar Radhakrishnan and Venkatesh Srinivasan},
  journal={Proceedings. Fourteenth Annual IEEE Conference on Computational Complexity (Formerly: Structure in Complexity Theory Conference) (Cat.No.99CB36317)},
  year={1999},
  pages={7-},
}

@inproceedings{
wei2022chain,
title={Chain of Thought Prompting Elicits Reasoning in Large Language Models},
author={Jason Wei and Xuezhi Wang and Dale Schuurmans and Maarten Bosma and brian ichter and Fei Xia and Ed H. Chi and Quoc V Le and Denny Zhou},
booktitle={Advances in Neural Information Processing Systems},
editor={Alice H. Oh and Alekh Agarwal and Danielle Belgrave and Kyunghyun Cho},
year={2022},
}

@article{feng2023towards,
  title={Towards revealing the mystery behind chain of thought: a theoretical perspective},
  author={Feng, Guhao and Zhang, Bohang and Gu, Yuntian and Ye, Haotian and He, Di and Wang, Liwei},
  journal={Advances in Neural Information Processing Systems},
  volume={36},
  pages={70757--70798},
  year={2023}
}

@article{nye2021show,
  title={Show your work: Scratchpads for intermediate computation with language models},
  author={Nye, Maxwell and Andreassen, Anders Johan and Gur-Ari, Guy and Michalewski, Henryk and Austin, Jacob and Bieber, David and Dohan, David and Lewkowycz, Aitor and Bosma, Maarten and Luan, David and others},
  journal = {ICLR 2022 Workshop DL4C},
  year={2022}
}

@inproceedings{smith2023s5,
  title={Simplified State Space Layers for Sequence Modeling},
  author={Smith, Jimmy TH and Warrington, Andrew and Linderman, Scott},
  booktitle={The Eleventh International Conference on Learning Representations},
  year      = {2023},
}

@inproceedings{orvieto2023resurrecting,
  title={Resurrecting recurrent neural networks for long sequences},
  author={Orvieto, Antonio and Smith, Samuel L and Gu, Albert and Fernando, Anushan and Gulcehre, Caglar and Pascanu, Razvan and De, Soham},
  booktitle={International conference on machine learning},
  pages={26670--26698},
  year={2023},
}

@article{de2024griffin,
  title={Griffin: Mixing Gated Linear Recurrences with Local Attention for Efficient Language Models},
  author={Soham De and Samuel L. Smith and Anushan Fernando and Aleksandar Botev and George Cristian-Muraru and Albert Gu and Ruba Haroun and Leonard Berrada and Yutian Chen and Srivatsan Srinivasan and Guillaume Desjardins and Arnaud Doucet and David Budden and Yee Whye Teh and Razvan Pascanu and Nando de Freitas and Caglar Gulcehre},
  journal={ArXiv},
  year={2024},
  volume={abs/2402.19427},
}

@inproceedings{peng2023rwkv,
  title={Rwkv: Reinventing rnns for the transformer era},
  author={Peng, Bo and Alcaide, Eric and Anthony, Quentin and Albalak, Alon and Arcadinho, Samuel and Biderman, Stella and Cao, Huanqi and Cheng, Xin and Chung, Michael and Derczynski, Leon and others},
  booktitle={Findings of the association for computational linguistics: EMNLP 2023},
  pages={14048--14077},
  year={2023}
}

@inproceedings{huang2025cot,
  title={Transformers Provably Learn Chain-of-Thought Reasoning with Length Generalization},
  author={Huang, Yu and Wen, Zixin and Singh, Aarti and Chi, Yuejie and Chen, Yuxin},
  booktitle={The Thirty-ninth Annual Conference on Neural Information Processing Systems},
  year      = {2025},
}

@article{siegelmann1995computational,
  title={On the Computational Power of Neural Nets},
  author={Siegelmann, HT and Sontag, ED},
  journal={Journal of Computer and System Sciences},
  volume={50},
  number={1},
  pages={132--150},
  year={1995},
}

@inproceedings{weiss2018practical,
  title={On the practical computational power of finite precision RNNs for language recognition},
  author={Weiss, Gail and Goldberg, Yoav and Yahav, Eran},
  booktitle={Proceedings of the 56th Annual Meeting of the Association for Computational Linguistics (Volume 2: Short Papers)},
  pages={740--745},
  year={2018}
}

@inproceedings{merrill2020formal,
  title={A formal hierarchy of RNN architectures},
  author={Merrill, William and Weiss, Gail and Goldberg, Yoav and Schwartz, Roy and Smith, Noah A and Yahav, Eran},
  booktitle={Proceedings of the 58th Annual Meeting of the Association for Computational Linguistics},
  pages={443--459},
  year={2020}
}

@inproceedings{deletang2023neural,
  title={Neural Networks and the Chomsky Hierarchy},
  author={Deletang, Gregoire and Ruoss, Anian and Grau-Moya, Jordi and Genewein, Tim and Wenliang, Li Kevin and Catt, Elliot and Cundy, Chris and Hutter, Marcus and Legg, Shane and Veness, Joel and others},
  booktitle={The Eleventh International Conference on Learning Representations},
  year      = {2023},
}

@article{merrill2022saturated,
  title={Saturated transformers are constant-depth threshold circuits},
  author={Merrill, William and Sabharwal, Ashish and Smith, Noah A},
  journal={Transactions of the Association for Computational Linguistics},
  volume={10},
  pages={843--856},
  year={2022}
}

@article{merrill2023parallelism,
  title={The parallelism tradeoff: Limitations of log-precision transformers},
  author={Merrill, William and Sabharwal, Ashish},
  journal={Transactions of the Association for Computational Linguistics},
  volume={11},
  pages={531--545},
  year={2023}
}

@inproceedings{jelassi2024repeat,
  title={Repeat after me: transformers are better than state space models at copying},
  author={Jelassi, Samy and Brandfonbrener, David and Kakade, Sham M and Malach, Eran},
  booktitle={Proceedings of the 41st International Conference on Machine Learning},
  pages={21502--21521},
  year={2024}
}

@inproceedings{arora2024zoology,
  title={Zoology: Measuring and Improving Recall in Efficient Language Models},
  author={Arora, Simran and Eyuboglu, Sabri and Timalsina, Aman and Johnson, Isys and Poli, Michael and Zou, James and Rudra, Atri and R{\'e}, Christopher},
  booktitle={Proceedings of 12th International Conference on Learning Representations (ICLR)},
  year={2024},
}

@book{kushilevitz1997communication,
  title     = {Communication Complexity},
  author    = {Kushilevitz, Eyal and Nisan, Noam},
  publisher = {Cambridge University Press},
  year      = {1997},
}

@article{alon1999space,
  title={The Space Complexity of Approximating the Frequency Moments},
  author={Alon, Noga and Matias, Yossi and Szegedy, Mario},
  journal={Journal of Computer and System Sciences},
  volume={58},
  number={1},
  pages={137--147},
  year={1999},
}

@inproceedings{telgarsky2016benefits,
  title={Benefits of depth in neural networks},
  author={Telgarsky, Matus},
  booktitle={Conference on learning theory},
  pages={1517--1539},
  year={2016}
}

@inproceedings{
wen2024rnns,
title={{RNN}s are not Transformers (Yet):  The Key Bottleneck on In-Context Retrieval},
author={Kaiyue Wen and Xingyu Dang and Kaifeng Lyu},
booktitle={The Thirteenth International Conference on Learning Representations},
year={2025},
}

@article{soydan2024s,
  title={S7: Selective and simplified state space layers for sequence modeling},
  author={Soydan, Taylan and Zubi{\'c}, Nikola and Messikommer, Nico and Mishra, Siddhartha and Scaramuzza, Davide},
  journal={arXiv preprint arXiv:2410.03464},
  year={2024}
}

@inproceedings{Zubic_2025_CVPR,
  title={Gg-ssms: Graph-generating state space models},
  author={Zubic, Nikola and Scaramuzza, Davide},
  booktitle={Proceedings of the Computer Vision and Pattern Recognition Conference},
  pages={28863--28873},
  year={2025}
}

@article{
zubic2025regularity,
title={Regularity and Stability Properties of Selective {SSM}s with Discontinuous Gating},
author={Nikola Zubic and Davide Scaramuzza},
journal={Transactions on Machine Learning Research},
issn={2835-8856},
year={2026},
}

@inproceedings{
zubic2025limits,
title={Limits of Deep Learning: Sequence Modeling through the Lens of Complexity Theory},
author={Nikola Zubic and Federico Sold{\`a} and Aurelio Sulser and Davide Scaramuzza},
booktitle={The Thirteenth International Conference on Learning Representations},
year={2025},
}

@inproceedings{
goyal2024think,
title={Think before you speak: Training Language Models With Pause Tokens},
author={Sachin Goyal and Ziwei Ji and Ankit Singh Rawat and Aditya Krishna Menon and Sanjiv Kumar and Vaishnavh Nagarajan},
booktitle={The Twelfth International Conference on Learning Representations},
year={2024},
}

@article{abbe2023generalization,
  title={Generalization on the Unseen, Logic Reasoning and Degree Curriculum},
  author={Emmanuel Abbe and Samy Bengio and Aryo Lotfi and Kevin Rizk},
  journal={International Conference on Machine Learning (ICML)},
  year={2023},
}

@inproceedings{
barak2022hidden,
title={Hidden Progress in Deep Learning: {SGD} Learns Parities Near the Computational Limit},
author={Boaz Barak and Benjamin L. Edelman and Surbhi Goel and Sham M. Kakade and Eran Malach and Cyril Zhang},
booktitle={Advances in Neural Information Processing Systems},
year={2022},
}

@inproceedings{chen2025theoretical,
  title={Theoretical limitations of multi-layer transformer},
  author={Chen, Lijie and Peng, Binghui and Wu, Hongxun},
  booktitle={2025 IEEE 66th Annual Symposium on Foundations of Computer Science (FOCS)},
  pages={2631--2653},
  year={2025},
  organization={IEEE}
}
\bibliographystyle{tmlr}

\end{document}